\def\neurips{1}
\def\arXiv{0}
\def\temp{0}

\def\version{\arXiv}

\date{}
\title{\textbf{Overparameterization from Computational~Constraints}}

\ifnum\version=\arXiv    

\documentclass[11pt]{article}
\usepackage{breakcites}
\usepackage{natbib}
\usepackage{paralist}
\usepackage{fullpage}
\usepackage{algorithm}
\usepackage{lipsum}
\usepackage{algorithmic}
\usepackage[margin=1in]{geometry}

\usepackage[utf8]{inputenc} 
\usepackage[T1]{fontenc}    
\usepackage[colorlinks,linkcolor=blue,urlcolor=blue,citecolor=blue]{hyperref}       
\usepackage{url}            
\usepackage{booktabs}       
\usepackage{amsfonts}       
\usepackage{nicefrac}       
\usepackage{microtype}      
\usepackage{xcolor} 
\usepackage{amsmath}
\usepackage{amssymb}
\usepackage{amsthm}
\usepackage{appendix}
\usepackage{bbm}
\usepackage{color}
\usepackage{mathtools}
\usepackage{xspace}
\usepackage{comment}
\newcommand{\Risk}{\mathrm{Risk}}
\newcommand{\Supp}{\mathrm{Supp}}
\newcommand{\cF}{\mathcal{F}}
\newcommand{\cA}{\mathcal{A}}
\newcommand{\parag}[1]{\textbf{#1}}

\newcounter{thm}

\newtheorem{theorem}[thm]{Theorem}
\newtheorem*{theorem*}{Theorem}
\newtheorem{construct}[thm]{Construction}
\newtheorem{lemma}[thm]{Lemma}

\newtheorem{claim}[thm]{Claim}
\newtheorem{fact}[thm]{Fact}

\newtheorem{definition}[thm]{Definition}

\let\leq\leqslant
\let\geq\geqslant

\newcommand{\remove}[1]{}

\newcommand{\cU}{\mathcal{U}}
\newcommand{\cX}{\mathcal{X}}
\newcommand{\cY}{\mathcal{Y}}

\newcommand{\cD}{\mathcal{D}}
\newcommand{\cH}{\mathcal{H}}
\newcommand{\cS}{\mathcal{S}}
\newcommand{\bbF}{\mathbb{F}}
\newcommand{\bbN}{\mathbb{N}}
\newcommand{\cT}{\mathcal{T}}
\DeclareMathOperator*{\Ex}{\mathbf{E}}

\newcommand{\Mnote}[1]{ { \textcolor{magenta}{[{\bf Mohammad: #1]}} } }
\newcommand{\mingyuan}[1]{{\color{blue}Mingyuan:  {#1}}}

\newcommand{\set}[1]{\{#1\}}

\newcommand{\tr}[1]{\mathsf{Tr}\left({#1}\right)}

\newcommand{\zo}{\{0,1\}}
\newcommand{\eps}{\varepsilon}
\newcommand{\Enc}{\mathsf{Enc}}
\newcommand{\LEnc}{\mathsf{LEnc}}

\newcommand{\poly}{\mathsf{poly}}
\newcommand{\samp}{\mathsf{samp}}
\newcommand{\negl}{\mathsf{negl}}
\newcommand{\draw}{\leftarrow}
\newcommand{\gen}{\mathsf{Gen}}
\newcommand{\sign}{\mathsf{Sign}}
\newcommand{\verify}{\mathsf{Verify}}
\newcommand{\vk}{\mathsf{vk}}
\newcommand{\sk}{\mathsf{sk}}

\newcommand{\bias}{\mathsf{bias}}
\newcommand{\sd}[1]{\mathsf{SD}\left({#1}\right)}
\providecommand{\defeq}[0]{\ensuremath{{\;\vcentcolon=\;}}\xspace}
\newcommand{\abs}[1]{\left\lvert{#1}\right\rvert}
\newcommand{\tuple}[1]{\left\langle{#1}\right\rangle}
\newcommand{\HD}{\mathsf{HD}}

\DeclareMathOperator*{\pr}{Pr}
\newcommand{\prob}[1]{\pr\left[{#1}\right]}
\newcommand{\probsub}[2]{\pr_{{#1}}\left[{#2}\right]}
\DeclareMathOperator*{\E}{E}

\newcommand{\expsub}[2]{\E_{{#1}}\left[{#2}\right]}

\renewcommand{\sim}{\draw}

\author{
Sanjam Garg\thanks{UC Berkeley and NTT Research \href{mailto:sanjamg@berkeley.edu}{sanjamg@berkeley.edu}} \and
Somesh Jha\thanks{University of Wisconsin, Madison \href{mailto:jha@cs.wisc.edu}{jha@cs.wisc.edu}} \and 
Saeed Mahloujifar\thanks{Princeton University \href{mailto:sfar@princeton.edu}{sfar@princeton.edu}} \and 
Mohammad Mahmoody\thanks{University of Virginia \href{mailto:mohammad@virginia.edu}{mohammad@virginia.edu}} \and
Mingyuan Wang\thanks{UC Berkeley \href{mailto:mingyuan@berkeley.edu}{mingyuan@berkeley.edu}}
}

\begin{document}
\maketitle

\begin{abstract} 

    Overparameterized models with millions of parameters have been hugely successful. In this work, we ask:  can the need for large models be, at least in part, due to the \emph{computational} limitations of the learner? Additionally, we ask, is this situation exacerbated for \emph{robust} learning? We show that this indeed could be the case. We show learning tasks for which computationally bounded learners need \emph{significantly more} model parameters than what information-theoretic learners need. Furthermore, we show that even more model parameters could be necessary for robust learning. In particular, for computationally bounded learners, we extend the recent result of Bubeck and Sellke [NeurIPS'2021] which shows that robust models might need more parameters, to the computational regime  and show that bounded learners could provably need an even larger number of parameters. 
     Then, we address the following related question: can we hope to remedy the situation for robust computationally bounded learning by  restricting \emph{adversaries} to also be computationally bounded for sake of obtaining  models with fewer parameters? Here again, we show that this could  be possible. Specifically, building on the work of Garg, Jha, Mahloujifar, and Mahmoody [ALT'2020], we demonstrate a learning task that can be learned  efficiently and robustly against a computationally bounded attacker, while to be robust against an information-theoretic attacker requires the learner to utilize significantly more parameters.

    %
\end{abstract}

\newpage
\tableofcontents
\newpage

\section{Introduction}

In recent years,  deep neural nets with millions or even billions of parameters\footnote{See \url{https://paperswithcode.com/sota/image-classification-on-imagenet}  for the size of the most   successful models for image classification of Imagenet.}~\citep{wortsman2022model,dai2021coatnet,yu2022coca} have emerged as one of the most powerful models for very basic tasks such as image classification. A magic of DNNs is that they generalize without falling into the classical theories mentioned above, and hence they are the subject of an active line of work aiming to understand how DNNs generalize~\cite{arora2019fine,allen2019convergence,allen2019learning,novak2018sensitivity,neyshabur2014search,kawaguchi2017generalization,zhang2021understanding,arora2018stronger} and the various benefits of overparametrized regimes~\cite{xu2018benefits,chang2020provable,arora2018optimization,du2018gradient,lee2019wide}.  In fact, the number of parameters in such models is so large that it is enough to memorize (and fit)  to a large number of \emph{random labels}~\citep{zhang2021understanding}. This leads us to our first main question, in which we investigate the potential cause for having large models: 

\begin{quote}
    \emph{Are there any learning tasks for which computationally bounded learners need to utilize significantly more model parameters than needed information-theoretically?}
\end{quote}

One should be cautious in how to formulate the question above. That is because, many simple (information-theoretically learnable) tasks are believed to be computationally hard to learn (e.g., learning parity with noise~\citep{pietrzak2012cryptography}). In that case, one can interpret this as saying that the efficient learner requires \emph{infinite} number of parameters, as a way of saying that the learning is not possible at all! However, as explained above, we are interested in understanding the reason behind needing a large number of model parameters when learning \emph{is possible}.

\parag{Could robustness also be a cause for overparameterization?} A highly sought-after property of machine learning models is their \emph{robustness} to the so-called adversarial examples~\citep{goodfellow2014explaining}. Here we would like to find a model $f$ such that $f(x) = f(x')$ holds with high probability when $x \gets D$ is an honestly sampled instance and $x' \approx x$ is a \emph{close} instance that is perhaps minimally (yet adversarially) perturbed.\footnote{The closeness here could mean that a human cannot tell the difference between the two images $x,x'$.}  A recent work of~\citet{bubeck2021universal} showed that having large model parameters \emph{could} be due to the robustness of the model. Here, we are asking whether such phenomenon can have a computational variant that perhaps leads to needing \emph{even more} parameters when the robust learner is running in polynomial time.

\begin{quote}
    \emph{Are there any learning tasks for which computationally bounded \textbf{robust} learning comes at the cost of having even more model parameters than   non-robust learning?}
\end{quote}

In fact, it was shown by \citet{bubeck2019adversarial} and Degwaker et al. \citep{degwekar2019computational} that computational limitations of the learner \emph{could} be the reason behind the vulnerability of models to adversarial examples. 
In this work, we ask whether the phenomenon investigated by the prior works is also crucial when the model size is considered.

\parag{Can computational intractability of adversary help?} We ask whether natural restrictions on \emph{the adversary} can help reduce model sizes. In particular, we consider the restriction to the class of polynomially bounded adversaries.
In fact, when it comes to robust learning, the computationally bounded aspect could be imposed both on the learner as well as the \emph{adversary}. 
\begin{quote}
    \emph{Are there any learning tasks for which   {robust} learning can be done with fewer model parameters when dealing with \textbf{polynomial-time} adversaries?}
\end{quote}

Previously, it was shown that indeed working with \emph{computationally bounded} adversaries \emph{can} help achieving robust models \citep{mahloujifar2019can,bubeck2019adversarial,garg2020adversarially}. Hence, we are asking a similar question in the context of model parameters.



\subsection{Our results}
  In summary, we prove that under the computational assumption that one-way functions exist, the answer to all three of our main questions above is positive. Indeed, we show that the computational efficiency of the learner could be the cause of having  overparametarized models. 
  {  Furthermore, the computational efficiency of the adversary could reduce the size of the models.}
  In particular, we prove the following theorem, in which a learning task is parameterized by $\lambda$, a hypothesis class $\cH$ and a class of input distributions $\cD$ (see Section \ref{sec:learn-def} for formal definitions).

\begin{theorem}[Main results, informal] \label{thm:main-inf}
If one-way functions exist, then for  arbitrary polynomials $n<\alpha<\beta$ (e.g., $n=\lambda^{0.1},\alpha=\lambda^{5},\beta=\lambda^{10}$)  over   $\lambda$
the following hold. 
\begin{itemize}
\item \textbf{Part 1:} There is a learning task parameterized by $\lambda$
and a robustness bound (to limit how much an adversary can perturb the inputs)  such that:
\begin{itemize}
    \item The instance  size is $\Theta(n)$. (That is, the length of the input is $\Theta(n)$.)
    \item There is a  robust learner that uses $\Theta(\lambda)$ parameters.
    \item Any \emph{polynomial-time} learner needs $\Theta(\alpha)$ parameters to learn the task.
    \item Any \emph{polynomial-time} learner needs $\Theta(\beta)$ parameters to robustly learn the task.
\end{itemize}
\item \textbf{Part 2:}  There is a learning task parameterized by $\lambda$ and a robustness bound (to limit how much an adversary can perturb the inputs) such that:
\begin{itemize}
    \item The  instance size is $\Theta(n)$.

    \item When the adversary that generates the adversarial examples runs in polynomial time, there is an (efficient) learner that outputs a model with $O(1)$ parameters that robustly predicts the output labels with small error.
    
    \item Against information-theoretic (computationally unbounded) adversaries, no learner  can produce a model with $<\Theta(\alpha)$   parameters that later robustly make the predictions.

\end{itemize}
\end{itemize}
\end{theorem} 

\ifnum\version=\neurips

In Sections \ref{sec:Part1} and \ref{sec:Part2}, we  explain the high-level ideas behind the proofs of the two parts of  Theorem \ref{thm:main-inf} and formalize these statements. 
The full proofs can be found in the supplemental material.

\else

In Section~\ref{sec:overview}, we present the high-level ideas behind the proofs of the two parts of  Theorem~\ref{thm:main-inf}.
The formal constructions and proofs can be found in Section~\ref{app:eff-learn} and Section~\ref{app:eff-adv}, respectively.

\fi

\ifnum\temp=0

\parag{Takeaway.}  Here we put our work in perspective.
As discussed above, prior works have considered the effect of computational efficiency (for both the learner
and the attacker%
) on the robustness of the model. Informally, these works have shown that requiring a learner to be efficient hinders robustness, while requiring an attacker to be efficient helps achieve robustness.
Our work   studies the effect of computational efficiency as well but focuses on the number of parameters of the model. In spirit, we have shown a similar phenomenon. Namely, requiring a learner to be efficient increases the size of the model, while requiring an attacker to be efficient helps reduce the size of the model. Our results can be summarized as follows.
  In the non-robust case, requiring the learner to be efficient increases the size of the model.
In the robust case:
(1) Requiring the learner to be efficient increases the size of the model. This holds for robustness against both efficient and inefficient attackers.
(2) Restricting to only computational efficient attackers reduces the size of the model. This holds for both efficient and inefficient learners.

\paragraph{Limitations, Implications, and Open Question.}
Our work shows that the phenomenon of having larger models due to computational efficiency could provably happen in certain scenarios. It does not imply, however, that this holds for all learning problems.
It is a fascinating open question whether similar phenomena also happen to real-world problems.
We note that this is not particular to our work, but common to most prior works in the theory of learning showing ``separation'' results~\citep{bubeck2019adversarial, degwekar2019computational,mahloujifar2019can,garg2020adversarially}.

Our results provides an explanation on why small but representative classes (e.g. 2 layers neural networks) of functions do not obtain same (robust) accuracy as larger models. This phenomenon that cannot be solely explained based on representation power of the function class might be due to computational limitations of the learning algorithm.

Finally, we note that our theorem demonstrates the separation by using the simplest setting of binary output. One can extend it to more sophisticated settings of any finite output, particularly real numbers with bounded precision. However, our work does not consider real numbers with infinite precision. In such cases, one needs to revisit computational efficiency, as the inputs are infinitely long. 

\fi

\section{Technical overview}
\label{sec:overview}
\ifnum\version=\neurips
\section{Efficient learner vs. information-theoretic learner} 
\else
\subsection{Efficient learner vs. information-theoretic learner}
\fi
\label{sec:Part1}

\ifnum\version=\neurips

In this section, we explain the ideas behind the proof of Part 1 of Theorem \ref{thm:main-inf}. 
We start with the high-level ideas behind our construction. Then, we formally state the construction of the learning task and  its properties through formal statements.
The proofs can be found in the supplemental material.

\else

In this section, we explain the ideas behind the proof of Part 1 of Theorem \ref{thm:main-inf}. 
The formal construction and the theorems of this result are deferred to Section~\ref{app:eff-learn}.

\fi

Consider the problem of learning an inner product function $\mathsf{IP}_P$ defined as $\mathsf{IP}_P(x)=\tuple{x,P}$, where the inner product is done in $\mathbb{GF}_2$.  Our first observation is that to learn $\mathsf{IP}_P$ with a small error where $P$ is uniformly random, the number of model parameters that the learner employs must be as (almost) as large  as the length of $P$. 
Intuitively, one can argue it as follows. Let $Z$ denote the parameters in the model that the learner outputs. Suppose that $Z$ is shorter than $P$. Then $P$ must still be unpredictable given $Z$.%
\footnote{That is, with high probability over the choice of $Z=z$, the conditional distribution $P\vert (Z=z)$ has high min-entropy. More formally, this unpredictability is measured by the average-case min-entropy (Definition~\ref{def:avg-min-ent}).}
By  a standard result in randomness extraction, one can argue that, for a uniform $x$,
$$(Z,x,\tuple{x,P})\approx (Z,x,U_{\zo}).$$
That is, the label $\tuple{x,P}$ looks information-theoretically uniform to the classifier who holds $Z$ and $x$. Therefore, the classifier can only output the correct label with (the trivial) probability $\approx 1/2$.

The conclusion is that learning \emph{all} linear functions need a learner that outputs as many parameters as the function's description is. However, even an \emph{efficient} learner can perform the learning just as well as an information-theoretic one (e.g., using the Gaussian elimination). We now show how to modify this    task   to make a big  difference between an efficient learner and an information-theoretic learner in terms of the number of parameters that they output. 

\parag{Computational perspective.} Now, suppose $P$ is \emph{computationally pseudorandom}~\citep{GM84} rather than being truly random.\footnote{A pseudorandom string is   indistinguishable from  random ones for computationally bounded distinguishers.} That is, let $f:\zo^\lambda\to\zo^\alpha$ be a pseudorandom generator (Definition~\ref{def:prg}) and $P$ is distributed as $f(U_\lambda)$, where $U_\lambda$ denotes the uniform distribution over $\lambda$ bits.
Since (1) an efficient learner cannot distinguish $P\draw f(U_\lambda)$ from $P\draw U_\alpha$ and (2) any learner who tries to learn $\mathsf{IP}_P$ with $P\draw U_\alpha$ needs $\alpha$ parameters, it can be proved that an efficient learner who tries to learn $\mathsf{IP}_P$ with $P\draw f(U_\lambda)$ must also uses $\alpha$ parameters.
However, an information-theoretic learner needs only $\lambda$ parameters to learn $\mathsf{IP}_P$ with $P\draw f(U_\lambda)$ as it can essentially find the seed $s$ such that $P=f(s)$ and output the seed $s$.
To conclude, to learn $\mathsf{IP}_P$ for $P\draw f(U_\lambda)$, an information-theoretic learner only needs a few (i.e., $\lambda$) parameters and an efficient learner needs a lot of (i.e., $\alpha$) parameters. We emphasize that for a pseudorandom generator, its output length $\alpha$ could have an arbitrarily large polynomial dependence on its input length $\lambda$. For instance, it could be $\alpha=\lambda^{10}$.

\parag{Robustness.} We now explain the ideas behind Theorem~\ref{thm:eff-robust-learn'} in which we study the role of model robustness  in the size of the model parameters. We now suppose the instance is sampled according to the distribution $D_Q$ (parametrized by a string $Q$), which is defined by
$$\Enc(U) + Q.$$
Here, $U$ is the uniform distribution (over the right number of bits),  $\Enc(U)$ is an error-correcting encoding of $U$, and the addition is coordinate-wise field addition. Moreover, the label for this instance is the inner product $\tuple{U,P}$ for some vector $P$.
Observe that if the learner learns $Q$, it can always find the correct label on a perturbed instance due to the error-correcting property of $\Enc(U)$.
We argue that, in order to robustly learn this task for a uniformly random $Q$, the number of parameters in the model that the learner outputs must be almost as large as the length of $Q$.
Intuitively, the argument is as follows. Let $Z$ be the model that the learner outputs. Since $|Z|<|Q|$, then $Q$ must still be unpredictable given $Z$. In this case, we prove that
$$(Z,\Enc(U)+Q+\rho) \approx (Z,U').$$
Here, $\rho$ stands for the noise that the adversary adds to the instance. In words, the classifier holding the parameter $Z$ cannot distinguish the perturbed instance $\Enc(U)+Q+\rho$ from a uniformly random string $U'$. Since $U$ is information-theoretically hidden to the classifier, it can only output the correct label $\tuple{U,P}$ with probability $\approx 1/2$.

Next, to explore the difference between an efficient and information-theoretic learner, we consider the case where $Q$ is pseudorandomly distributed, i.e. $Q\draw f(U_\lambda)$ for some $f:\zo^\lambda\to\zo^\beta$.
Again, since (1) an efficient learner cannot distinguish $Q\draw U_\beta$ from $Q\draw f(U_\lambda)$ and (2) any learner needs at least $\approx \abs Q$ parameters to learn the task with $Q\draw U_\beta$, an efficient learner must also need at least $\approx \abs Q$ parameters to learn the task.
On the other hand, an information-theoretic learner could again find the seed $s$ such that $Q=f(s)$ and output the seed $s$ as the parameter.
To conclude, an information-theoretic learner requires few (i.e., $\lambda$) parameters to robustly learn the task and an efficient learner needs a lot of (i.e., $\approx \beta$) parameters to robustly learn the task. (Again, $\beta$ could have an arbitrary polynomial dependence on $\lambda$.)

\parag{Making instances small.}
The learning tasks we considered above suffer from one drawback: the size of the instance is very large, or at least is related to the number of parameters of the model.  Here we ask: is it possible to have a small instance size while an efficient learner still needs to output a very large model? We   answer this question positively. 
In particular, for any $n=\poly(\lambda)$ (e.g., $n=\lambda^{0.1}$), we construct a learning problem where the instance size is $\Theta(n)$ and the efficient learner still needs $\alpha$ (resp. $\beta$) parameters to learn (resp.  robustly learn) the task. Our construction uses the ``sampler'' by~\citet{C:Vadhan03}. Informally, a sampler $\samp$ (see Lemma~\ref{thm:samp}) needs a small seed $u$ and outputs a subset of $\{1,2,\ldots,\alpha\}$ of size $n$. The sampler comes with the guarantee   that if $P$ is a source with high entropy, $P\vert_{\samp(u)}$ also has high enough entropy. To illustrate the usage of the sampler, consider learning this new inner product function $\mathsf{IP}_P$ defined as
$$\mathsf{IP}_P(u,x)=\tuple{x,P\vert_{\samp(u)}}.$$
For uniformly random $P$, one can similarly argue that a model must output at least $\abs{P}$ parameters to learn the task.
Let $Z$ denote the parameters in the model that the learner outputs. If $\abs{Z}<\abs{P}$, then $P$ contains high entropy conditioned on $Z$. By the property of the sampler, it must hold that $P\vert_{\samp(u)}$ contains high entropy conditioned on $Z$ and $u$. Consequently,
$$(Z,u,x,\mathsf{IP}_P(u,x))\approx (Z,u,x,U_{\zo}).$$
Namely, the classifier who sees the parameter $Z$ and the instance $(u,x)$ can only predict the label with probability $\approx 1/2$.
Observe that the size of the instance $(u,x)$ is roughly $n$,%
\footnote{As the seed $u$ is very small.}
while $P$ could have length $\alpha\gg n$.
The use of the sampler in the robust learning case is similar to the non-robust case above. We refer the readers to \ifnum\version=\neurips Appendix~\ref{app:eff-learn} \else Section~\ref{app:eff-learn} \fi for more details.


\ifnum\version=\neurips
To sum up, we present our formal construction and formal theorem statement below. The formal proof is presented in the supplemental material.

\begin{construct}[Parameter-heavy models under efficient learning] \label{const:1'}
Given the parameter $n<\lambda<\alpha<\beta$, we construct the following learning problem.%
\footnote{All the other parameters are implicitly defined by these parameters.}
We rely on the following building blocks. 
\begin{itemize}
        \item Let $f_1\colon\zo^\lambda\to\zo^\alpha$ and $f_2\colon\zo^\lambda\to\zo^\beta$ be PRGs (Definition~\ref{def:prg}).
        \item Let $\Enc$ be a RS encoding with dimension $k$, block length $n$, and rate $R=k/n$. The rate is chosen to be any constant $<1/3$ and $k$ is defined by $R$ and $n$.
    This RS code is over the field $\bbF_{2^\ell}$ for some $\ell=\Theta(\log \lambda)$.
    \item Let $\samp_1\colon\zo^{r_1}\to\binom{\{1,\ldots,\alpha\}}{{k\cdot\ell}}$ and $\samp_2\colon\zo^{r_2}\to \binom{\{1,\ldots,\beta\}}{{n\cdot\ell}}$ be samplers. We obtain these samplers by invoking Lemma~\ref{thm:samp} with sufficiently small $\kappa_1$ and $\kappa_2$ (e.g.,  $\kappa_1=\Theta(1/\log\lambda)$ and sufficiently small constant $\kappa_2$). Note that $\kappa_1,\kappa_2$  define $r_1,r_2$.

    \item  For any binary string $v$, we use $[v]$ for an arbitrary error-correcting encoding of $v$ (over the field $\bbF_{2^\ell} $) such that $[v]$ can correct $>(1-R)n/2$ errors. This can always be done by encoding $v$ using RS code with a suitable  (depending on the dimension of $v$) rate.
    Looking forward, we shall consider an adversary that may perturb $\leq (1-R)n/2$ symbols. Therefore, when a string $v$ is encoded as $[v]$ and the adversary perturb it to be $\widetilde{[v]}$, it will always be error-corrected and decoded back to $v$.


\end{itemize}
 We now construct the following learning task $F_\lambda=(\cX_\lambda,\cY_\lambda,\cD_\lambda,\cH_\lambda)$.
\begin{itemize}
\item $\cX_\lambda$   implicitly defined by the distribution $\cD_\lambda$, and $\cY_\lambda=\zo$.
\item $\cD_\lambda$ consists of distributions $D_s$ for  $s\in\zo^\lambda$, where $D_s$ is 
$$D_s=\Big([u_1],\; [u_2],\; m,\; \Enc(m)+ \left(f_2(s)\big\vert_{\mathsf{samp}_2(u_2)}\right)\Big), \text{ where }$$
\begin{itemize}
    \item $u_1$ and $u_2$ are sampled uniformly at random from $\zo^{r_1}$ and $\zo^{r_2}$, respectively.
    \item $m$ is sampled uniformly at random from $\bbF_{2^\ell}^k$.
    \item $f_2(s)\big\vert_{\mathsf{samp_2}(u_2)}$ is interpet as a vector over $\bbF_{2^\ell}$ and $\Enc(m)+ \left(f_2(s)\big\vert_{\mathsf{samp_2}(u_2)}\right)$ is coordinate-wise addition over $\bbF_{2^\ell}$.
\end{itemize}

\item $\cH_\lambda$ consists of all functions $h_s:\cX\to\cY=\zo$ for all $s\in\zo^\lambda$, where $h_s$ is:
$$   h_s(x) = \left\langle m, \left(f_1(s)\big\vert_{\mathsf{samp}_1(u_1)}\right) \right\rangle  ,$$
and
the inner product  is over $\bbF_2$ and $m$ is interpreted as a string $\in\bbF_2^{k\cdot\ell}$ in the natural way.
\item {\bfseries Adversary.} The entire input $x= ([u_1], [u_2],  m, \Enc(m)+ \left(f_2(s) \vert_{\mathsf{samp_1}(u_1)}\right))$ is interpreted as a vector over $\bbF_{2^\ell}$ and we consider an adversary that may perturb $\leq (1-R)n/2$ symbols. That is,   the adversary has a budget of $(1-R)n/2$ in Hamming distance over $\bbF_{2^\ell}$.
\end{itemize}


\end{construct}

\begin{theorem} \label{thm:1'}
An information-theoretic learner can (robustly) $\eps$-learn the task of Construction \ref{const:1'} with parameter size $2\lambda$ and sample complexity $ \Theta(\frac\lambda\eps )$.
Moreover, an efficient learner can (robustly) $\eps$-learn this task with parameter size $\alpha+\beta$ and sample complexity  $ \Theta(\frac{\alpha+\beta}\eps )$. 
\end{theorem}

\begin{theorem}\label{thm:eff-learn'}
Any efficient learner that outputs models with $\leq \alpha/2$ parameters cannot $\eps$-learn $F_\lambda$ of Construction \ref{const:1'} for $\eps<1/3$.
\end{theorem}

\begin{theorem}\label{thm:eff-robust-learn'}
There exists some constant $c$ such that the following holds. In the presence of an adversary that may perturb $(1-R)n/2$ symbols, any efficient learner for the task of Construction \ref{const:1'} that outputs a model with $c\cdot \beta /\log\lambda$ parameters cannot $\eps$-robustly learn $F_\lambda$ for $\eps<1/3$.
\end{theorem}

\else
\fi
\ifnum\version=\neurips
\section{Efficient adversary vs. information-theoretic adversary} 
\else
\subsection{Efficient adversary vs. information-theoretic adversary} 
\fi
\label{sec:Part2}

\ifnum\version=\neurips

In this section, we explain the ideas behind the proof of Part 2 of Theorem \ref{thm:main-inf}. At the end of this section, we formally state the construction of the learning task and state its properties through formal statements. The proofs can be found in the supplemental material.

\else

In this section, we explain the ideas behind the proof of Part 2 of Theorem \ref{thm:main-inf}. 
The formal construction and the theorems of this result is deferred to Section~\ref{app:eff-adv}.

\fi

We now explore whether the computational efficiency of the adversary could affect the number of parameters required to robustly learn a task.
\cite{GJMM20} considered the difference between an efficient adversary and an information-theoretic one in terms of their running time.
We first recall their construction.
The learning instance is sampled from the distribution
$$[\vk],b,\sign(\sk,b).$$
Here, $(\vk,\sk)$ is the verification key and signing key pair of a signature scheme (Definition~\ref{def:sign});%
\footnote{
Every instance samples a fresh pair of verification key and signing key $(\vk,\sk)$.
}
$[\vk]$ is an error-correcting encoding of $\vk$, which ensures that $[\vk]$ can always be recovered after perturbation by the adversary; $b$ is a uniform random bit and the label of the instance is simply $b$.

The signature scheme ensures that an efficient adversary cannot forge a valid signature and, hence, any efficient adversary that perturbs $(b,\sign(\sk,b))$ will result in an invalid message/signature pair. 
A classifier could then detect such perturbations and output a special symbol $\bot$ indicating that perturbation is detected.
On the other hand, an information-theoretic adversary could launch a successful attack by forging a valid signature $(1-b,\sign(\sk,1-b))$.%
Therefore, it could perturb the instance to be $[\vk],1- b ,\sign(\sk,1-b)$ and, hence, flipping the label of the output.

\parag{First idea.} We want to construct a learning problem such that the learner needs few parameters against an efficient adversary, but a lot of parameters against an information-theoretic adversary. Our first idea is to add another way of recovering $b$ in the above learning problem. Consider the learning problem where the instance is sampled from distribution $D_P$ defined as
$$[\vk],b,\sign(\sk,b),[b+\tuple{u,P}],[u].$$
Here, $u$ is uniformly distributed.
Observe that if the learner learns $P$, it can   recover $b$ correctly from $[b+\tuple{u,P}]$ and $[u]$ by  error-correction decoding.
However, if the number of parameters in the model that the learner outputs have $<\abs{P}$ parameters, then $[b+\tuple{u,P}]$ and $[u]$ information-theoretically hides $b$. Again, this is because $P$ is unpredictable given the parameter $Z$%
\footnote{That is, with high probability over the choice of $Z=z$, the conditional distribution $P\vert (Z=z)$ has high min-entropy. More formally, this unpredictability is measured by the average-case min-entropy (Definition~\ref{def:avg-min-ent}).}
and, hence, 
$$Z,u,\tuple{u,P}\approx Z,u,U_{\zo}.$$
Consequently, an information-theoretic adversary could again launch the attack that replace $b,\sign(\sk,b)$ with $1-b,\sign(\sk,1-b)$ and successfully flipping the label.
Therefore, a learner must employ $\abs{P}$ parameters to robustly learn the task against information-theoretic adversaries.

\parag{Second idea.} In the above learning task, a learner employing $\abs P$ parameters can always recover the correct label against information-theoretic adversaries. However, against efficient adversaries, a learner with fewer parameters may not always recover the correct label but will sometimes output the special symbol $\bot$ indicating that tampering is detected. Could we twist it to ensure that a learner with fewer parameters can also always recover the correct label against efficient adversaries? We positively answer this question by using list-decodable code (Definition~\ref{def:list-code}). Intuitively, list-decoding ensures that given an erroneous codeword, the decoding algorithm will find the list of all messages whose encoding is close enough to the erroneous codeword.

Our new learning task has instances drawing from the distribution $D_P$ defined as
$$[\vk],\LEnc(b,\sign(\sk,b)),[b+\tuple{u,P}],[u].$$
Here, $\LEnc$ is the encoding algorithm of the list-decoding code.
The main idea is that: after the perturbation on $\LEnc(b,\sign(\sk,b))$, the original message/signature pair $b,\sign(\sk,b)$ will always appear in the list output by the decoding algorithm.
Now, against an efficient adversary, $b,\sign(\sk,b)$ must be the only valid message/signature pair in the list. Otherwise, this adversary breaks the unforgeability of the signature scheme.
Against an information-theoretic adversary, however, it could introduce $(1-b,\sign(\sk,1-b))$ into the list recovered by the list-decoding algorithm.
Consequently, the learner cannot tell if the correct label is $b$ or $1-b$.
Consequently, for this learning task, against information-theoretic adversaries, one still needs $\abs{P}$ parameters. Against efficient adversaries, one needs only a few parameters and can always recover the correct label.

\parag{Making instances small.} Again, we have this unsatisfying feature that the instance has the same size as the model. We resolve this issue using the sampler in a similar way. We refer the readers to \ifnum\version=\neurips Appendix~\ref{app:eff-adv} \else Section~\ref{app:eff-adv} \fi for more details.

\ifnum\version=\neurips
We present our formal construction and theorem statement below. The formal proof is  in Appendix~\ref{app:eff-adv}.

\begin{construct}[Learning task for bounded/unbounded attackers] \label{const:2'}
Given the parameters $n<\lambda<\alpha$, we construct the following learning problem.%
\footnote{All the other parameters are implicitly defined by these parameters.}
We use  the following tools. 
\begin{itemize}
        \item $(\gen,\sign,\verify)$ be a signature scheme (see Definition~\ref{def:sign}).
        \item Let $\LEnc$ be a RS encoding with dimension $k$, block length $n$, and rate $R=k/n$. We pick the rate $R$ to be any constant $<1/4$ and $k$ is defined by $R$ and $n$.
    This RS code is over the field $\bbF_{2^\ell}$ for some $\ell=\Theta(\log \lambda)$.

    \item Let $\samp \colon\zo^{r}\to\binom{\{1,\ldots,\alpha\} }{n}$  be samplers. (We obtain these samplers by invoking Lemma~\ref{thm:samp} with sufficiently small $\kappa_1$ and $\kappa_2$. For instance, setting $\kappa_1=\Theta(1/\log\lambda)$ and $\kappa_2$ to be any small constant suffices.)

    \item  For any binary string $v$, we use $[v]$ for an arbitrary error-correcting encoding of $v$ (over the field $\bbF_{2^\ell} $) such that $[v]$ can correct $>(1-\sqrt{R})n$ errors. This can always be done by encoding $v$ using RS code with a suitable  (depending on the dimension of $v$) rate.
    Looking forward, we shall consider an adversary that may perturb $\leq (1-\sqrt R)n$ symbols. Therefore, when a string $v$ is encoded as $[v]$ and the adversary perturbs it to be $\widetilde{[v]}$, it will always be error-corrected and decoded back to $v$.
    

\end{itemize}

 We now construct the following learning task $F_\lambda=(\cX_\lambda,\cY_\lambda,\cD_\lambda,\cH_\lambda)$.
\begin{itemize}
    \item $\cX_\lambda$ is $\zo^N$ for some $N$ that is implicitly defined by $\cD_\lambda$, and $\cY_\lambda$ is $\zo$.
    \item The distribution $\cD_\lambda$ consists of all distribution $D_s$ for $s\in\zo^\alpha$
$$D_s =  \Big([u],[v],\; [\vk],\; \LEnc(b,\sign(\sk,b)),\; \left[b+\tuple{v,s\vert_{\samp(u)}}\right]\Big), \text{ such that }$$
\begin{itemize}
    \item $u$ are sampled uniformly  from $\zo^r$ and $v$ is sampled uniformly  from $\zo^n$.

    \item $(\vk,\sk)\draw\gen(1^\lambda)$ are sampled from the signature scheme.
    
    \item $b$ is sampled uniformly at random from $\zo$. $(b,\sign(\sk,b))$ is intepreted as a vector in $\bbF_{2^\ell}^k$ in the natural way.%

\end{itemize}
    \item  $h_\lambda$ consists of one single function $h$. On input $x=([u],[v],\LEnc(b,\sign(\sk,b),[b+\tuple{v,s\vert_\samp(u)}])$, $h(x)$ simply decodes $\LEnc(b,\sign(\sk,b))$ and output $b$.
    \item {\bfseries Adversary.} The entire input $ ([u],[v], [vk],  \LEnc(b,\sign(\sk,b)),  [b+\langle v,s\vert_{\samp(u)}\rangle ])$ is interpreted as a vector over $\bbF_{2^\ell}$ and we consider an adversary that may perturb $\leq (1-\sqrt R)n$ symbols. That is, the   adversary has a budget of $(1-\sqrt R)n$ for Hamming distance over $\bbF_{2^\ell}$.
\end{itemize}


\end{construct}

\begin{theorem} \label{thm:4-1'}
For the learning task of Construction \ref{const:2'}, there is an efficient learner (with 0 sample complexity) that outputs a model with no parameter and $\negl(\lambda)$-robustly learns $F_\lambda$ against computationally-bounded adversaries of budget $(1-\sqrt R)n$.
\end{theorem}

\begin{theorem} \label{thm:4-2'}
For computationally unbounded adversaries, any information-theoretic learner with $\alpha/2$ parameters cannot $\eps$-robustly learn $F_\lambda$ for $\eps<1/3$ for the learning task of Construction \ref{const:2'}.
\end{theorem}

\else

\fi

\section{Preliminaries}


For a distribution $D$, $x\draw D$ denotes that $x$ is sampled according to $D$.
We use $U_\cS$ for the uniform distribution over $\mathcal{S}$, and we use $U_n$ to represent $U_{\zo^n}$. For a set $\cS$, $s\draw \cS$ means $s \draw \cU_\cS$.
For any two distributions $X$ and $Y$ over a finite universe $\Omega$, their {\em statistical distance}  is defined as
$$\sd{X,Y}\defeq \frac12\cdot\sum_{\omega\in\Omega}\abs{\prob{X=\omega}-\prob{Y=\omega}}.$$
We sometimes write $X\approx_\eps Y$ to denote $\sd{X,  Y}\leq \eps$.
For a vector $x=(x_1,\ldots,x_n)\in\bbF^n$ and a subset $\cS=\{i_1,\ldots,i_\ell\}\subseteq\{1,2,\ldots,n\}$, $x_\cS$ denotes $(x_{i_1},\ldots,x_{i_\ell})$.
For any two vectors $x=(x_1,\ldots,x_n)$ and $y=(y_1,\ldots,y_n)$, their {\em Hamming distance} is defined as $\HD(x,y)=\abs{\left\{i\in\{1,2,\ldots,n\}\,:\, x_i\neq y_i\right\}}$.
For a set $\cS$ and an integer $0\leq t\leq\abs{\cS}$, $\binom{\cS}{t}$ represents the set of subsets of $\cS$ of size $t$.
$\mathbb{I}$ stands for the indicator function.
The base for all logarithms in this paper is 2.

\subsection{Definitions related to learning and attacks} \label{sec:learn-def}
In this subsection, we present   notions and definitions that are related to learning.

We use $\cX$ to denote the inputs and $\cY=\zo$ to denote the outputs or the labels. By $\cH$ we denote a \emph{hypothesis class} of functions from $\cX$ to $\cY$.  We use $D$ to denote a distribution over $\cX \times \cY$. A learning algorithm,   takes a set $\cS \in (\cX \times \cY)^*$ and a parameter $\lambda$ and outputs a function $f$ that is supposed to predict a fresh sample from the same distribution that has generated the set $\cS$. The parameter $\lambda$ is supposed to capture the complexity of the problem, e.g., by allowing the inputs (and the running times) to grow.  (E.g., this could be the VC dimension, but not necessarily so.) In particular, we assume that the members of the sets $\cX_\lambda,\cY_\lambda$ can be represented with $\poly(\lambda)$ bits. A \emph{proper} learning for a hypothesis class $\cH$ outputs $f \in \cH$, while an \emph{improper} learner is allowed to output arbitrary functions.
By default, we work with improper learning as we do not particularly focus on whether the learning algorithms output a function from the hypothesis class.
%
For a set $\cS \subseteq \cX$ and a function $h \colon \cX \to \cY$, by $\cS^h$ we denote the \emph{labeled} set $\set{(x,h(x)) \mid x \in \cS}$.
For a distribution $D$ and an oracle-aided algorithm $A$, by $A^D$ we denote giving $A$ access to a \emph{$D$ sampler}.

\begin{definition}[Learning problems and learners]
We use $\cF=\set{F_\lambda=(\cX_\lambda, \cY_\lambda, \cD_\lambda,\cH_\lambda)}_{\lambda\in\mathbb{N}}$ to denote a family of learning problems where each $\cH_\lambda$ is a set of hypothesis functions mapping $\cX_\lambda$ to $\cY_\lambda$ and $\cD_\lambda$ is a set of distributions supported on  $\cX_\lambda$. 

\begin{itemize}
\item For function $\eps(\cdot,\cdot)$, we say $L$ $\eps$-learns $\cF$ if 
$$\forall \lambda \in \mathbb{N},  D_\lambda \in \cD_\lambda,  h\in \cH_\lambda,  n \in \mathbb{N};\   \Ex_{\cS \sim D_\lambda^n; f \sim L(\cS^h, \lambda)}[ \Risk(h, D_\lambda,f)]\leq \eps(\lambda,n).$$
where   $\Risk(h,D_\lambda,f)= \Pr_{x\sim D_\lambda}[h(x)\neq f(x)]$.

\item $L$ outputs models with at most $p(\cdot)$ bits of parameters if
$\forall \lambda\in \mathbb{N};\  {\left|\Supp\big(L(\cdot, \lambda)\big)\right|}\leq 2^{p(\lambda)}.$

\item $L$ is an $\eps$-robust learner against (all) adversaries of budget $r$ w.r.t. distance metric $d$ if
$$\forall \lambda \in \mathbb{N},  D_\lambda \in \cD_\lambda,  h\in \cH_\lambda,  n \in \mathbb{N};   \Ex_{\cS \sim D_\lambda^n; f \sim L(\cS^h, \lambda)}[ \Risk_{d,r}(h, D_\lambda,f)]\leq \eps(\lambda,n),$$
$$\text{ where }\Risk_{(d,r)}(h,D_\lambda,f) = \Pr_{x\sim D_\lambda}[\max_{x'} \mathbb{I}\big(f(x')\neq h(x)\big)\land \mathbb{I}\big(d(x,x')\leq r\big)].$$ 

\item   $L$ runs in polynomial time if there is a polynomial $\poly(\cdot)$ such that for all $\lambda \in \mathbb{N}$ and $\cS\in (\cX_\lambda,\cY_\lambda)^*$, the running time of $L(\cS,\lambda)$ is bounded by $\poly(|\cS|\cdot  \lambda)$.
\item $L$ is an $\eps$-robust learner against    polynomial-time adversaries of budget $r$ w.r.t distance $d$, if for any family of  $\poly(\lambda)$-time (oracle aided) adversaries $\cA=\set{A^{(\cdot)}_\lambda}_{\lambda \in \mathbb{N}}$ we have
$$\forall \lambda \in \mathbb{N},  D_\lambda \in \cD_\lambda,  h\in \cH_\lambda,  n \in \mathbb{N};  \Ex_{\cS \sim D_\lambda^n; f \sim L(\cS^h, \lambda)}[ \Risk_{d,r, A_\lambda}(h, D_\lambda,f)]\leq \eps(\lambda,n), $$
$$\text{where } \Risk_{(d,r,A_\lambda)}(h,D_\lambda,f) = \Pr_{x\sim D_\lambda, x'\sim A^{D_\lambda}(x,h(x),f)}[ \mathbb{I}\big(h(x')\neq f(x)\big)\land \mathbb{I}\big(d(x,x')\leq r\big)].$$ 


\end{itemize}
\end{definition}

\ifnum\version=\neurips
\else
Our proof also relies on the following theorem from \cite{bubeck2019adversarial}.  

\begin{theorem}[Implied by Theorem 3.1 of \cite{bubeck2019adversarial}]\label{thm:learnable}
 Let $\{\mathcal{C}_\lambda\}_\lambda$ be a finite family of classifiers. 
Suppose the learning problem $\cF=\{\cF_\lambda\}_\lambda$ and a learner $L $ satisfy the following. For all $D_\lambda\in\cD_\lambda$ and $h \in\cH_\lambda$, and sample $\cS \gets D^n_\lambda$, $L(\cS^h)$ always outputs a classifier   $f\in\mathcal{C}_\lambda$ such that $\Risk_{d,r}(h,D_\lambda,f)=0$ (i.e., $f$ robustly fits $h$ perfectly). Then, $L$ will $\delta$-robust learn $\cF$ with sample complexity $ \log\abs{\mathcal{C}_\lambda}/\delta$.
\end{theorem}
We emphasize that in the theorem above, one might pick a \emph{different}   set of classifiers merely for sake of computational efficiency of the learner $L$. Namely, it might be possible to information-theoretically learn a hypothesis class robustly (e.g., by a robust variant of empirical risk minimization when), but an efficient learner might choose to output its classifiers from a larger set such that it can \emph{efficiently} find a member of that class.
\fi

\subsection{Cryptographic primitives}

\begin{definition}[Negligible function]
A function $\eps(\lambda)$ is said to be negligible, denoted by $\eps(\lambda)=\negl(\lambda)$, if for all polynomial $\poly(\lambda)$, it holds that
$\eps(\lambda)<1/\poly(\lambda)$
for all sufficiently large $\lambda$.
\end{definition}

\begin{definition}[Pseudorandom generator]\label{def:prg}
Suppose $\alpha>\lambda$. A function $f\colon\zo^\lambda\to\zo^\alpha$ is called a pseudorandom generator (PRG) if for all polynomial-time probabilistic algorithm $A$, it holds that
$$\abs{\probsub{x\draw  {\zo^\lambda}}{A(f(x))=1}-\probsub{y\draw  {\zo^\alpha}}{A(y)=1}}=\negl(\lambda).$$
The ratio $\alpha/\lambda$ is refer to as the (multiplicative) stretch of the PRG.
\end{definition}


\begin{lemma}[\cite{hill99}]
Pseudorandom generators (with arbitrary polynomial stretch) can be constructed from any one-way functions.%
\footnote{A one-way function is a function that is easy to compute, but hard to invert (see Definition~\ref{def:owf}).}
\end{lemma}

\begin{definition}[Signature scheme]\label{def:sign}
A \emph{signature} scheme consists of three algorithms $(\gen,\sign,\verify)$.
\begin{itemize}
    \item $(\vk,\sk)\draw\gen(1^\lambda)$: on input the security parameter $1^\lambda$, the randomized algorithm $\gen$ outputs a (public) verification key $\vk$ and a (private) signing key $\sk$.
    \item $\sigma=\sign(\sk,m)$: on input the signing key $\sk$ and a message $m$, $\sign$ outputs a signature $\sigma$.
    \item $b=\verify(\vk,m,\sigma)$: on input a verification key $\vk$, a message $m$, and a signature $\sigma$, $\verify$ outputs a bit $b$, and $b=1$  indicates tha the signature is accepted.
\end{itemize}
We require a (secure) signature scheme to satisfy the following properties.
\begin{itemize}
    \item {\bfseries Correctness.} For every message $m$, it holds that 
    $$\prob{ \verify(\vk,m,\sigma)=1 \ \colon\  \begin{gathered}
       (\vk,\sk)\draw\gen(1^\lambda)\\
       \sigma = \sign(\sk,m)
    \end{gathered}} = 1.$$
    \item {\bfseries Weak unforgeability.\footnote{We consider a rather weak notion of unforgability. Here, we require that the adversary cannot forge a signature when he is given only one valid pair of message and signature. This weaker security notion already suffices for our purposes. In the stronger notion, the adversary is allowed to pick $m$ based on the given $\vk$.}} For any PPT algorithm $A$ and message $m$,  it holds that
    $$\prob{ \begin{gathered}
       m'\neq m \text{ and }\\
       \verify(\vk,m',\sigma')=1
    \end{gathered} \ \colon\  \begin{gathered}
        (\vk,\sk)\draw\gen(1^\lambda)\\
       \sigma = \sign(\sk,m)\\
       (m',\sigma')\draw A(\vk,m,\sigma)
    \end{gathered}} =\negl(\lambda).$$
    
\end{itemize}
\end{definition}

\begin{lemma}[\cite{STOC:NaoYun89,STOC:Rompel90}]
Signature schemes can be constructed from any one-way function.
\end{lemma}

\subsection{Coding theory}
Let $\bbF$ be a finite field.
A {\em linear code} $\cH$ with {\em block length} $n$ and {\em dimension} $k$ is a $k$-dimensional subspace of the vector space $\bbF^n$. 
The {\em generator matrix} $G\in\bbF^{k\times n}$ maps every message $m\in\bbF^k$ to its encoding, namely, the message $m$ is encoded as $m\cdot G$.
The {\em distance} $d$ of the code $\cH$ is the minimum distance between any two codewords. 
That is, $d = \min\limits_{(x,y\in\cH)\wedge (x\neq y)}\HD(x,y)$. The {\em rate} of the codeword is defined as $R=k/n$.

\parag{Reed-Solomon code.} In this work, we will mainly use the Reed-Solomon (RS) code. For the RS code, every message $m=(m_1,\dots,m_k)\in\bbF^k$ is parsed as a degree $k-1$ polynomial $$f(x)= m_1+m_2\cdot x+\cdots+m_{k}\cdot x^{k-1}$$
and the encoding is simply $(f(1),f(2),\ldots,f(n))\in\bbF^n$. 

\begin{definition}[List-decodable code] \label{def:list-code}
A code $\cH\subseteq \bbF^n$ is said to be $(p,L)$-list-decodable if for any string $x\in\bbF^n$, there are $\leq L$ messages whose encoding $\mathsf{c}$ satisfies $\HD(x,\mathsf{c})\leq p\cdot n$.
We say the code $\cH$ is efficiently $(p,L)$-list-decodable if there is an efficient algorithm that finds all such messages.
\end{definition}

\begin{lemma}[\cite{FOCS:GurSud98}]
For any constant $R>0$, the Reed-Solomon code with constant rate $R$ and block length $n$ is efficiently $(1-\sqrt R,\poly(n))$-list-decodable. 
\end{lemma}

\ifnum\temp=0

\subsection{Randomness extraction}

\begin{definition}[Min-entropy]
The min-entropy of a distribution $X$ over a finite set $\Omega$ is defined as
$$H_\infty(X)\defeq -\log\left(\max_{\omega\in\Omega}\ \prob{X=\omega}\right).$$
\end{definition}

We need the following result about the sampler.
A sampler (for a target set of coordinates $\{1,2,\ldots,n\}$) is a deterministic mapping that only takes randomness as input and outputs a subset of  $\{1,2,\ldots,n\}$. Below, we let $\binom{[n]}{t}=\set{\cS\mid |\cS|=t, \cS \subseteq [n]}$.

\begin{lemma}[Lemma 6.2 and Lemma 8.4 of~\cite{Vadhan04}]\label{thm:samp}
For any $0<\kappa_1,\kappa_2<1$ and any $n,t\in\bbN$, there exists a deterministic  sampler $\samp\colon \zo^r\to \binom{[n]}{t}$ 
such that the following hold.
\begin{itemize}
    \item If $X$ is a random variable over $\zo^n$ with min-entropy   $\geq\mu\cdot n$, there exists a  random variable $Y$ over $\zo^t$ with min-entropy   $\geq (\mu-\kappa_1)\cdot t$ and  
     $$\sd{\;\big(U_r,X_{\samp\left(U_{r}\right)}\big) \;,\; \big(U_{r},Y\big)\;}\leq \exp(-\Theta(n\kappa_1))+\exp\left(-n^{\kappa_2}\right),$$
    where the two $U_r$ in the first joint distribution refer to the same sample.
    
    \item Furthermore, $r = \Theta\left(n^{\kappa_2}\right)$.
\end{itemize}
\end{lemma}

This lemma by Vadhan~\cite{Vadhan04} states that one could use a small amount of randomness to sub-sample from a distribution $X$ with the guarantee that $X_{\samp(U_{ r})}$ is close to another distribution $Y$ that preserves the same min-entropy rate as $X$. 
 
\fi

\ifnum\version=\neurips

Due to space limitations, additional preliminaries are included in Appendix~\ref{app:add-prelim}.

\else

\ifnum\temp=1

\begin{definition}[Min-entropy]
The min-entropy of a distribution $X$ over a finite universe $\Omega$ is defined as
$$H_\infty(X)\defeq -\log\left(\max_{\omega\in\Omega}\ \prob{X=\omega}\right).$$
\end{definition}

We need the following result about the sampler.
A sampler (for a target set of coordinates $\{1,2,\ldots,n\}$) is a deterministic mapping that only takes randomness as input and outputs a subset of  $\{1,2,\ldots,n\}$. Below, we let $\binom{[n]}{t}=\set{\cS\mid |\cS|=t, \cS \subseteq [n]}$.

\begin{theorem}[Lemma 6.2 and Lemma 8.4 of~\cite{Vadhan04}]\label{thm:samp}
For any $0<\kappa_1,\kappa_2<1$ and any $n,t\in\bbN$, there exists a deterministic  sampler $\samp\colon \zo^r\to \binom{[n]}{t}$ 
such that the following hold.
\begin{itemize}
    \item If $X$ is a random variable over $\zo^n$ with min-entropy   $\geq\mu\cdot n$, there exists a  random variable $Y$ over $\zo^t$ with min-entropy   $\geq (\mu-\kappa_1)\cdot t$ and  
     $$\sd{\;\big(U_r,X_{\samp\left(U_{r}\right)}\big) \;,\; \big(U_{r},Y\big)\;}\leq \exp(-\Theta(n\kappa_1))+\exp\left(-n^{\kappa_2}\right),$$
    where the two $U_r$ in the first joint distribution refer to the same sample.
    
    \item Furthermore, $r = \Theta\left(n^{\kappa_2}\right)$.
\end{itemize}
\end{theorem}

This theorem by Vadhan~\cite{Vadhan04} states that one could use a small amount of randomness to sub-sample from a distribution $X$ with the guarantee that $X_{\samp(U_{ r})}$ is close to another distribution $Y$ that preserves the same min-entropy rate as $X$. 
 
\fi

In certain cases, some information regarding $X$ is learned (e.g., through training). Let us denote this learned information as a random variable $Z$. Note that $X$ and $Z$ are two correlated distributions.
In order to denote the min-entropy of $X$ conditioned on the learned information $Z$, we need the following notion (and lemma) introduced by Dodis et. al.~\cite{DORS08}.

\begin{definition}[Average-case min-entropy~\cite{DORS08}]\label{def:avg-min-ent}
For two correlated distributions $X$ and $Z$, the average-case min-entropy is defined as
$$
\widetilde{H}_\infty(X | Z) = - \log\left(\expsub{z \draw Z}{\max_{x}\; \prob{X = x |Z = z} }\right).
$$
\end{definition}

\begin{lemma}[\cite{DORS08}]\label{lem:min-ent}
We have the following two lemmas regarding the average-case min-entropy.
\begin{itemize}
    \item If the support set of $Z$ has size at most $2^m$, we have $$\widetilde{H}_\infty(X | Z)\geq H_\infty(X)-m.$$

    \item 
    It holds that 
    $$\probsub{z\draw Z}{H_\infty(X\vert Z=z)\geq \widetilde{H}_\infty(X |  Z)-\log(1/\eps)}\geq 1-\eps.$$
    
\end{itemize}
\end{lemma}

Intuitively, the above lemma states that: if a model denoted by the random variable $Z$ is not too large, and that model captures the information revealed about a variable $X$, since the support set of $Z$ has small size, then the average-case min-entropy  $\widetilde{H}_\infty(X|Z)$ is  large.
Furthermore, for most $z$, the min-entropy of $H_\infty(X\vert Z=z)$ is almost as large as $\widetilde{H}_\infty(X|Z)$.


We now recall a tool that, roughly speaking, states that if $X$ is a distribution over $\zo^n$ that contains some min-entropy, then the inner product (over $\bbF_2$) between $X$ and a random vector $Y$ is a uniformly random bit, even conditioned on most of $Y$.   This is a special case of the celebrated {\em leftover hash lemma}~\cite{hill99}. We summarize this result as the following theorem. For completeness, a proof can be found in Appendix~\ref{app:proof-ip}. 

\begin{theorem}[Inner product is a good randomness extractor]\label{thm:ip}
For all distribution $X$ over $\zo^n$ such that $H_\infty(X)\geq  2\cdot \log(1/\eps)$, it holds that
$$\left( U_n,\langle X, U_n\rangle\right) \;\approx_\eps\; U_{n+1},$$
where the two $U_n$ refer to the same sample. 
\end{theorem}

Next, we need the following notion  and results from Fourier analysis.

\begin{definition}[Small-bias distribution~\cite{NN93}]
Let $\bbF_{2^\ell}$ be the finite field of order $2^\ell$.
For a distribution $X$ over $\bbF_{2^\ell}^n$, the bias of $X$ with respect to a vector $ y\in\bbF_{2^\ell}^n$ is defined as
$$\bias(X, \alpha)\defeq \abs{\expsub{x\draw X}{(-1)^{\tr{\tuple{x,\alpha}}}}},$$ 
where $\mathsf{Tr}\colon\bbF_{2^\ell}\to\bbF_2$ denote the {\em trace map} defined as 
$\tr{y}=y+y^{2^1}+y^{2^2}+\cdots+y^{2^{\ell-1}}.$ 
The distribution $ X$ is said to be $\eps$-small-biased if for all non-zero vector $\alpha\in\zo^n$, it holds that
$$\bias( X,\alpha)\leq \eps.$$
\end{definition}
Note that the trace map maps elements from $\bbF_{2^\ell}$ to $\bbF_2$, where exactly half of the field elements maps to $1$ and the other half to $0$.
Consequently, if $\tuple{X,\alpha}$ is a uniform distribution over $\bbF_{2^\ell}$, then $\bias(X,\alpha)=0$.
In Appendix~\ref{app:masking-proof},  prove the theorem below.

\begin{theorem}[Small-biased Masking Lemma~\cite{STOC:DodSmi05}]\label{thm:masking}
Let $ X$ and $ Y$ be distributions over $\bbF_{2^\ell}^n$. If $H_\infty( X)\geq k$ and $Y$ is $\eps$-small-biased, it holds that
$$ \sd{ X + Y \;,\; U_{\zo^{n\ell}}}\leq 2^{\frac{n\ell-k}2-1}\cdot \eps.$$
\end{theorem}

Finally, we observe the following property about the noisy RS code. A proof can be found in Appendix~\ref{app:noisy-proof}.

\begin{theorem}[Noisy RS code is small-biased]\label{thm:noisy-code}
Let $\cH$ be a RS code over $\bbF_{2^\ell}$ with block length $n$ and rate $R$. For all integer $s\leq n$, consider the  following distribution
$$ D = \left\{\begin{gathered}
\mathsf{c}\draw\cH\\
\text{Sample a random }\cS\subseteq \{1,2,\ldots,n\}\text{ such that }\abs{\cS} = s\\
\forall i\in \cS,\text{ replace }\mathsf{c}_i\text{ with a random field element}\\
\text{Output }\mathsf{c}
\end{gathered}\right\}.$$
It holds that $D$ is $(1-R)^s$-small-biased.
\end{theorem}

\fi

\section{Efficient learning could need more model parameters}
\label{app:eff-learn}

In this section, we formally prove the first part of Theorem \ref{thm:main-inf}, which is the separation result between the number of parameters needed by unbounded v.s. bounded learners.

\remove{
\Mnote{Do we still need the informal result here?}
\mingyuan{probably not.}
\paragraph{The main result, informally.} Let $\lambda \in \mathbb{N}$ be the security parameter. Let $n<\alpha<\beta$ have arbitrary polynomial dependence on $\lambda$ (e.g., $\alpha = \lambda^5,\ \beta=\lambda^{10},\ n=\lambda^{0.1}$). Assume the existence of one-way functions, we shall construct a learning problem that, at a high level, satisfies the following properties.
\begin{enumerate}
    \item The instances are of size $\Theta(n\log\lambda)$.
    \item There exists an information-theoretic learner that employs $\lambda$ parameters to (robustly) learn the task.  
    \item Any efficient learner needs $\Omega(\alpha)$ parameters to learn the task.
    \item Any efficient learner needs $\Omega(\beta/\log \lambda) $ parameters to robustly learn the task.
\end{enumerate}
}

\ifnum\version=\neurips
Our construction of the learning problem in presented in Construction~\ref{const:1'}.
\else

Our construction and theorems are formally stated as follows.

\fi

First, observe that instance size is approximately $\Theta((k+n)\cdot \ell) =\Theta(n\cdot\log\lambda)$ as the size of the sampler inputs $u_1$ and $u_2$ are sufficiently small. 
Our theorems prove the following.
First, in Theorem \ref{thm:1'} we establish the efficient (robust) learnability of the task of Construction~\ref{const:1'}, where  the efficient-learner variant requires more parameters. Then, in Theorem~\ref{thm:eff-learn'} we establish the lower bound on the number of parameters needed by an efficient learner. Finally, in Theorem~\ref{thm:eff-robust-learn'} we establish the lower bound on the number of parameters needed by efficient \emph{robust} learners.

In the rest of this section, we prove these theorems.

\ifnum\version=\neurips

\begin{theorem*}[Restatement of Theorem~\ref{thm:1'}]
An information-theoretic learner can (robustly) $\eps$-learn the task of Construction \ref{const:1'} with parameter size $2\lambda$ and sample complexity $ \Theta(\frac\lambda\eps )$.
Moreover, an efficient learner can (robustly) $\eps$-learn this task with parameter size $\alpha+\beta$ and sample complexity  $ \Theta(\frac{\alpha+\beta}\eps )$. 
\end{theorem*}

\else

\subsection{Proof of Theorem~\ref{thm:1'}}

\fi

Since the learning task of Theorem \ref{thm:1'} has a finite hypothesis class, its learnability follows from the classical result of learning finite classes \citep{BenDBook}. Moreover, this can be done efficiently as this is a linear task. When it comes to learning \emph{robust} functions, one can also use the result of \cite{bubeck2019adversarial} for robustly learning finite classes.%
\footnote{Results for infinite classes could be found in subsequent works \cite{montasser2019vc}.} The learner of \cite{bubeck2019adversarial} simply uses the empirical-risk minimization, however this is done with respect to the \emph{robust} empirical risk. This learner is not always polynomial-time, even if the (regular) risk minimization can be done efficiently.
However, we would like to find \emph{robust} learners also \emph{efficiently}. The formal proof follows. 

\begin{proof}[Proof of Theorem \ref{thm:1'}]
Consider the set of functions $f_{s,s'}:\cX_\lambda\to\cY_\lambda$ for all $s,s'\in\zo^\lambda$ as follows.
\begin{enumerate}
    
    \item On input $x=(\widetilde{[u_1]},\widetilde{[u_2]},\widetilde{m},\widetilde{\mathsf{d}})$, it invokes the error-correcting decoding algorithm on $\widetilde{[u_1]}$ and $\widetilde{[u_2]}$ to find $u_1$ and $u_2$.
    
    \item 
    It uses $\widetilde{\mathsf{d}}+f_2(s')\vert_{\samp(u_2)}$ to get an encoding $\widetilde{\mathsf{c}}$ of $m$.
    
    \item It invokes the error-correcting decoding on $\widetilde{\mathsf{c}}$ to get $m$.

    \item It outputs $\tuple{m,f_1(s)\vert_{\samp(u_1)}}$.

\end{enumerate}

One of the function $f_{s,s'}$ will (perfectly) robustly fit the distribution since all the encodings $[u_1],[u_2],\Enc(m)$ tolerates $(1-R)n/2$ perturbation.
Since, there are $2^{2\lambda}$ such functions,
by Theorem~\ref{thm:learnable}, we conclude that an information-theoretic learner can $\eps$-learn this task with $2\lambda$ parameters and sample complexity $\Theta(\lambda/\eps)$.

Note that an efficient learner might not be able to find such a function $f_{s,s'}$ as it requires inverting a pseudorandom generator. However, an efficient learner can still learn using more samples and parameters as follows.
Consider the set of functions $f_{P,Q}:\cX_\lambda\to\cY_\lambda$ for all $P\in\zo^\alpha$ and $Q\in\zo^\beta$ as follows.
\begin{enumerate}
    
    \item On input $x=(\widetilde{[u_1]},\widetilde{[u_2]},\widetilde{m},\widetilde{\mathsf{d}})$, it invokes the error-correcting decoding algorithm on $\widetilde{[u_1]}$ and $\widetilde{[u_2]}$ to find $u_1$ and $u_2$.
    
    \item 
    It uses $\widetilde{\mathsf{d}}+Q\vert_{\samp(u_2)}$ to get an encoding $\widetilde{\mathsf{c}}$ of $m$.
    
    \item It invokes the error-correcting decoding on $\widetilde{\mathsf{c}}$ to get $m$.

    \item It outputs $\tuple{m,P\vert_{\samp(u_1)}}$.

\end{enumerate}

Similarly, one of the function $f_{P,Q}$ will (perfectly) robustly fit the distribution since all the encodings $[u_1],[u_2],\Enc(m)$ tolerates $(1-R)n/2$ perturbation.
Finding $f_{P,Q}$ that fits all the samples only requires linear operations and, hence, is efficiently learnable.
As there are $2^{\alpha+\beta}$ such functions, again
by Theorem~\ref{thm:learnable}, we conclude that an efficient learner can $\eps$-learn this task with $\alpha+\beta$ parameters and sample complexity $\Theta((\alpha+\beta)/\eps)$. To apply Theorem~\ref{thm:learnable}, we simply pretend that the hypothesis set is the larger set of $2^{\alpha+\beta}$ such functions, in which case our learner finds one of these $2^{\alpha+\beta}$ functions that perfectly matches with the training set with zero \emph{robust} empirical risk.
\end{proof}

\ifnum\version=\neurips

\begin{theorem*}[Restatement of Theorem~\ref{thm:eff-learn'}]
Any efficient learner that outputs models with $\leq \alpha/2$ parameters cannot $\eps$-learn $F_\lambda$ of Construction \ref{const:1'} for $\eps<1/3$.
\end{theorem*}

\else
\subsection{Proof of Theorem~\ref{thm:eff-learn'}}
\fi

\begin{proof}
We start by defining another learning problem $F_\lambda'$. This learning problem is identical to $F_\lambda$ for $\cX_\lambda$, $\cY_\lambda$, and $\cD_\lambda$. However, $\cH_\lambda'$  consists of all functions $h_P$ for all $P\in\zo^\alpha$, such that  
$$h_P(x) =  \left\langle m, \left(P\big\vert_{\mathsf{samp}_1(u_1)}\right) \right\rangle.$$
On a high level, our proof consists of two claims.
\begin{claim}\label{clm:eff-learn-1}
Fix any distribution $D_{s'}\in\cD_\lambda$. We consider a random hypothesis function $h_s$ and $h_P$, where $s\draw\zo^\lambda$ and $P\draw\zo^\alpha$. It holds that
    $$\Ex_{\cS \sim D_\lambda^n; f \sim L(\cS^{h_s}, \lambda)}[ \Risk(h_s, D_\lambda,f)]\quad \approx_{\negl(\lambda)}\quad  \Ex_{\cS \sim D_\lambda^n; f \sim L(\cS^{h_P}, \lambda)}[ \Risk(h_P, D_\lambda,f)].$$    
\end{claim}
\begin{claim}\label{clm:eff-learn-2}
For any learner $L$ (with an arbitrary sample complexity) with $\leq \alpha/2$ parameters, we have
    $$\Ex_{\cS \sim D_\lambda^n; f \sim L(\cS^{h_P}, \lambda)}[ \Risk(h_P, D_\lambda,f)]>3/8.$$      
\end{claim}

Note that, if both claims are correct, the theorem statement is true. 

We first show Claim~\ref{clm:eff-learn-1}. 
Observe that, given the string $P$, one can compute the function $h_P(x)$ efficiently. Now, given a string $P$, which is either a pseudorandom string (i.e., $P\draw f_1(U_\lambda)$) or a truly random string (i.e., $P\draw\zo^\alpha$). Consider the following distinguisher
$$ \left\{\begin{gathered}
\cS\draw D_{s'}^n,\  f\draw L(\cS^{h_P},\lambda),\ x\draw D_{s'}\\
\text{Output }\mathbb{I}(h_P(x)=f(x))
\end{gathered}\right\}.$$
If $P$ is pseudorandom, the probability that the distinguisher outputing 1 is $$\Ex_{\cS \sim D_\lambda^n; f \sim L(\cS^{h_s}, \lambda)}[ \Risk(h_s, D_\lambda,f)];$$ if $P$ is truly random, the probability that the distinguisher outputs 1 is $$\Ex_{\cS \sim D_\lambda^n; f \sim L(\cS^{h_P}, \lambda)}[ \Risk(h_P, D_\lambda,f)].$$ Therefore, if Claim~\ref{clm:eff-learn-1} does not hold, we break the pseudorandom property of the PRG.



It remains to prove Claim~\ref{clm:eff-learn-2}.

Since $P$ is sampled uniformly at random, we have $H_\infty(P)=\alpha$.
Let $Z$ denote the random variable $L(\cdot,\lambda)$, i.e., the model learned by the learner. Since the learner's output model employs $\leq \alpha/2$ parameters, we have $\Supp(Z)\leq 2^{\alpha/2}$.
And by Lemma~\ref{lem:min-ent}, we must have $$\widetilde{H}_\infty(P\vert Z)\geq \alpha/2.$$
Now, let us define the set%
\footnote{This set is called ``good'' as it is good for the learner.}
$$\mathsf{Good}=\{z\in\Supp(Z) \;:\; H_\infty(P \vert Z=z)\leq \alpha/4\}.$$
Lemma~\ref{lem:min-ent} implies that
$$\prob{Z\in\mathsf{Good}}\leq 2^{-\alpha/4}=\negl(\lambda).$$
In the rest of the analysis, we conditioned on the event that $Z\notin\mathsf{Good}$, which means $H_\infty(P\vert Z=z)>\alpha/4$.
Now, let $x= ([u_1], [u_2], m, \Enc(m)+  (f_2(s')\big\vert_{\mathsf{samp}_2(u_2)} ) ) $ be the test sample.
Since $P $ has min-entropy rate $> 1/4$, the property of the sampler (Lemma~\ref{thm:samp}) guarantees that there exists a distribution $D$ such that
$$H_\infty(D)\geq \left(\frac14-\kappa_1\right)\cdot k\ell>\frac15\cdot k\ell$$
and
$$\sd{(u_1, D) \ ,\  (u_1, P\vert_{\samp_1(u_1)}) }\leq \exp(-\Theta(\alpha\kappa_1))+\exp\left(-n^{\kappa_2}\right)=\negl(\lambda).$$
Recall that $x=([u_1], [u_2], m, \Enc(m)+  (f_2(s')\big\vert_{\mathsf{samp}_2(u_2)} ) )$ and $y=h_P(x) = \langle m,  (P\big\vert_{\mathsf{samp}_1(u_1)} )  \rangle $. Consequently,
\begin{align*}
    &\; \sd{((x,z),y) \;,\; ((x,z),U_{\zo})}\\
=&\; \sd{((u_1,m,z),y) \;,\; ((u_1,m,z),U_{\zo}}\tag{as $u_2$ is independent of $y$}\\
\leq&\; \sd{\Big((u_1,m,z), \tuple{m,D}\Big) \;,\; ((u_1,m,z), U_{\zo})} + \negl(\lambda) \tag{as $\sd{P_{\samp_1(u_1)}\vert Z=z \ ,\  D}\leq\negl(\lambda)$}\\
\leq&\; 2^{-k\ell/10} +\negl(\lambda) 
=\negl(\lambda). \tag{Theorem~\ref{thm:ip} and $H_\infty(D)>\frac15\cdot k\ell$}
\end{align*}
Therefore, in the learner's view $(x,z)$, $y$ is statistically $\negl(\lambda)$-close to uniform.
Therefore,
\begin{align*}
&\;\Ex_{\cS \sim D_\lambda^n; f \sim L(\cS^{h_P}, \lambda)}[ \Risk(h_P, D_\lambda,f)] \\
\geq&\; \Pr_{\cS \sim D_\lambda^n; f \sim L(\cS^{h_P}, \lambda)}[ z\in \mathsf{Good}] \\
&\; +\Pr_{\cS \sim D_\lambda^n; f \sim L(\cS^{h_P}, \lambda)}[ z\notin \mathsf{Good}]\cdot \Ex_{\cS \sim D_\lambda^n; f \sim L(\cS^{h_P}, \lambda)}[ \Risk(h_P, D_\lambda,f)\vert  z\notin \mathsf{Good}] \\
\geq&\; \negl(\lambda)+(1- \negl(\lambda))\cdot \left(\frac12 
 - \negl(\lambda)\right) >3/8.
\end{align*}
This shows Claim~\ref{clm:eff-learn-2} and completes the proof of the theorem.
\end{proof}

\ifnum\version=\neurips

\begin{theorem*}[Restatement of Theorem~\ref{thm:eff-robust-learn'}]
There exists some constant $c$ such that the following holds. In the presence of an adversary that may perturb $(1-R)n/2$ symbols, any efficient learner for the task of Construction \ref{const:1'} that outputs a model with $c\cdot \beta /\log\lambda$ parameters cannot $\eps$-robustly learn $F_\lambda$ for $\eps<1/3$.
\end{theorem*}

\else
\subsection{Proof of Theorem~\ref{thm:eff-robust-learn'}}
\fi

\begin{proof}

The high-level structure of the proof is similar to the proof of Theorem~\ref{thm:eff-learn'}. We consider a new learning problem $F_\lambda'$ that has the same $\cX_\lambda$, $\cY_\lambda$, and $\cH_\lambda$. However, $\cD_\lambda$ consists of all distribution $D_Q$ for all $Q\in\zo^\beta$, where the distribution $D_Q$ is
$$D_Q=\Big([u_1],\; [u_2],\; m,\; \Enc(m)+ \left(Q\big\vert_{\mathsf{samp}_2(u_2)}\right)\Big).$$
The proof consists of two claims.
\begin{claim}\label{clm:rob-learn-1}
Fix a hypothesis $h_{s'}\in\cH_\lambda$. We consider a random distribution over $\cD_\lambda$ and $\cD_\lambda'$. That is, $D_s$ and $D_Q$ are sampled with $s\draw\zo^\lambda$ and $Q\draw\zo^\beta$. It holds that
    $$\Ex_{\cS \sim D_s^n; f \sim L(\cS^{h_{s'}}, \lambda)}[ \Risk(h_{s'}, D_s,f)]\quad \approx_{\negl(\lambda)}\quad  \Ex_{\cS \sim D_Q^n; f \sim L(\cS^{h_{s'}}, \lambda)}[ \Risk(h_{s'}, D_Q,f)].$$    
\end{claim}
\begin{claim}\label{clm:rob-learn-2}
For any learner $L$ (with an arbitrary sample complexity) with $\leq c\cdot \beta/\log\lambda$ parameters, it holds that
    $$\Ex_{\cS \sim D_Q^n; f \sim L(\cS^{h_{s'}}, \lambda)}[ \Risk(h_{s'}, D_Q,f)]>3/8.$$     
\end{claim}

Note that these two claims prove the theorem. To see Claim~\ref{clm:rob-learn-1}, observe that given a string $Q$, we can sample efficiently from $D_Q$. 
Analogous to the proof of Claim~\ref{clm:eff-learn-1}, if Claim~\ref{clm:rob-learn-1} does not hold, we may break the pseudorandom property of the PRG using this (efficient) learner $L$.

It remains to prove Claim~\ref{clm:rob-learn-2}.

Let $Z$ denote the random variable $L(\cdot,\lambda)$, i.e., the parameter of the learner.
Since $Q$ is uniformly random, we have
$$H_\infty(Q\vert Z)\geq (1-c/\log\lambda)\beta.$$
Let us define the set
$$\mathsf{Good}=\{z\in\Supp(Z) \;:\; H_\infty(Q \vert Z=z)\leq (1-2c/\log\lambda)\beta\}.$$
Lemma~\ref{lem:min-ent} implies that
$$\prob{Z\in\mathsf{Good}}\leq 2^{-c\beta/\log\lambda}=\negl(\lambda).$$
In the rest of the analysis, we conditioned on the event that $Z\notin\mathsf{Good}$, which means $H_\infty(Q\vert Z=z)>(1-2c/\log\lambda)\beta$.

Now, we consider the following adversary $A$ that perturbs $(1-R)n/2$ symbols. Given a test instance $(x,y)$, where
$$x = \Big([u_1],\ [u_2],\ m,\ \Enc(m)+ \left(Q\big\vert_{\mathsf{samp_2}(u_2)}\right)\Big),$$
the adversary will do the following.
\begin{itemize}
    \item Replace $m$ with a uniformly random string. This costs a budget of $Rn$.

    \item Samples a random subset $\cT\subseteq\{1,2,\ldots,n\}$ of size $ (1-3R)n/{2}$. It adds noises to $\Enc(x)+ \left(Q\big\vert_{\mathsf{samp_2}(u_2)}\right)$ at precisely those indices from $S$. This costs a budget of $(1-3R)n/{2}$.

    For simplicity, let us denote the distribution of this noise by $\rho$. That is, $\rho$ is a distribution over $ \bbF_{2^\ell}^n$ such that it is 0 everywhere except for a random subset $\cT$ and for those $i\in \cT$, $\rho_i$ is uniformly random.
\end{itemize}

We now argue that the perturbed instance is statistically $\negl(\lambda)$-close to the distribution
$$\Big([u_1],\ [u_2],\ U_{{k\cdot\ell}},\  U_{{n\cdot\ell}}\Big).$$
It suffices to prove that $\Enc(m)+ \left(Q\big\vert_{\mathsf{samp_2}(u_2)}\right)+\rho$ is close to the uniform distribution.
By Theorem~\ref{thm:noisy-code}, $\Enc(m)+\rho$ is $(1-R)^{\frac{(1-3R)n}2}$-small-biased.

Furthermore, since $Q$ has min-entropy rate $> (1-2c/\log\lambda)$, the property of the sampler (Lemma~\ref{thm:samp}) guarantees that there exists a distribution $D$ such that
$$H_\infty(D)\geq (1-2c/\log\lambda-\kappa_1)\cdot n\ell>(1-3c/\log\lambda )\cdot n\ell.$$
and
\begin{equation}
\sd{(u_2, D) \ ,\  (u_2, Q\vert_{\samp_2(u_2)}) }\leq \exp(-\Theta(\beta \kappa_1))+\exp\left(-n^{\kappa_2}\right)=\negl(\lambda).\label{eq:1}
\end{equation}

Finally, by Theorem~\ref{thm:masking}, we have
$(\Enc(m)+\rho)+ D$ is
$$2^{\frac{3cn\ell}{2\log\lambda}-1}\cdot (1-R)^{\frac{(1-3R)n}2} $$
close to the uniform distribution.
Observe that as long as 
$$c < \frac{\log\lambda}{3\ell}\cdot(1-3R)\log(1/(1-R))=\Theta(1),$$
the closeness is negligible in $\lambda$.
Overall,
\begin{align*}
&\;\sd{\Enc(m)+Q\vert_{\samp_2(u_2)}+\rho\;,\; U_n}\\
\leq&\; \sd{\Enc(m)+Q\vert_{\samp_2(u_2)}+\rho\;,\;\Enc(m)+D+\rho}+ \sd{\Enc(m)+D+\rho\;,\; U_n}\tag{Triangle inequality}\\
\leq&\; \negl(\lambda)+\negl(\lambda) = \negl(\lambda).\tag{Equation~\ref{eq:1} and Theorem~\ref{thm:noisy-code} }
\end{align*}

Therefore, given a test instance $x$ and the perturbed input $x'$, we have
\begin{align*}
&\; \sd{((x',z),y) \ ,\  ((x',z),U_{\zo})}\\
\leq &\;\sd{((u_1,u_2,U_{k\ell},U_{n\ell}),y) \ ,\  ((u_1,u_2,U_{k\ell},U_{n\ell}),U_{\zo})} +\negl(\lambda)\\
=&\; \sd{\left(u_1,\tuple{m,f_1(s')\vert_{\samp_1(u_1)}}\right) \ ,\ (u_1,U_{\zo})}+\negl(\lambda)\\
=&\;\negl(\lambda).
\end{align*}
Hence, when $z\notin\mathsf{Good}$, given a perturbed test input $x'$, the correct label $y$ is information-theoretically unpredicatable from the learner with $\negl(\lambda)$ advantage.

Putting everything together, we have
\begin{align*}
&\;\Ex_{\cS \sim D_Q^n; f \sim L(\cS^{h_{s'}}, \lambda)}[ \Risk(h_{s'}, D_Q,f)] \\
\geq&\; \Pr_{\cS \sim D_Q^n; f \sim L(\cS^{h_{s'}}, \lambda)}[ z\in \mathsf{Good}] \\
&\; +\Pr_{\cS \sim D_Q^n; f \sim L(\cS^{h_{s'}}, \lambda)}[ z\notin \mathsf{Good}]\cdot \Ex_{\cS \sim D_Q^n; f \sim L(\cS^{h_{s'}}, \lambda)}[ \Risk(h_{s'}, D_Q,f)\vert  z\notin \mathsf{Good}] \\
\geq&\; \negl(\lambda)+(1- \negl(\lambda))\cdot \left(\frac12 
 - \negl(\lambda)\right) >3/8.
\end{align*}

This completes the proof of the claim and the entire theorem.
\end{proof}

\section{Computationally robust learning could need fewer parameters}
\label{app:eff-adv}
In this section, we formally prove Part 2 of Theorem \ref{thm:main-inf}. 
\ifnum\version=\neurips
Our construction of the learning problem is formally presented in Construction~\ref{const:2'}.
\else
Our construction and theorems are formally stated as follows.

\fi

We note that the instance size is (approximately) $\Theta(n \cdot \ell) =\Theta(n\cdot\log\lambda)$.
We shall prove two properties of this construction. In Theorem \ref{thm:4-1'}, we establish the \emph{upper bound} of learnability with few parameters under efficient (polynomial-time) attacks. Later, in Theorem \ref{thm:4-2'}, we establish the lower bound of the number of parameters when the attacker is unbounded.

In the rest of this section, we formally prove these theorems.

\ifnum\version=\neurips

\begin{theorem*}[Restatement of Theorem~\ref{thm:4-1'}]
For the learning task of Construction \ref{const:2'}, there is an efficient learner (with 0 sample complexity) that outputs a model with no parameter and $\negl(\lambda)$-robustly learns $F_\lambda$ against efficient adversaries of budget $(1-\sqrt R)n$.
\end{theorem*}

\else
\subsection{Proof of Theorem~\ref{thm:4-1'}}
\fi

\begin{proof}
The learner is defined as follows. On input a perturbed instance $x'=(\widetilde{[u]},\widetilde{[v]},\widetilde{[\vk]},\widetilde{\mathsf{c}},\widetilde{\mathsf{d}})$, it does the following:
\begin{enumerate}
    \item Invoke the error-correction algorithm to recover $\vk$.
    \item Invoke the list-decoding algorithm on $\widetilde{\mathsf{c}}$ to find a list of message/signature $(b_i,\sigma_i)$ pairs.
    \item Run the verifier to find   any valid message/signature pair $(b^*,\sigma^*)$ and output $b^*$. If no such pair exists, output a random bit, and if there are more than one such pair, pick one  arbitrarily.
\end{enumerate}

Observe that the learner can always recover the correct $\vk$ since the encoding  $[\vk]$ tolerates $(1-\sqrt R)n$ errors.

Next, suppose the original instance is $([u],[v], [vk],  \LEnc(b,\sign(\sk,b))$. Then, $(b,\sign(\sk,b))$ is always in the list of message/signature pairs output by the list-decoding algorithm. This is due to that $\LEnc(b,\sign(\sk,b))$ is $(1-\sqrt R)n$-close to the perturbed encoding $\widetilde{\mathsf{c}}$ and the list decoding algorithm outputs all such messages whose encoding is $(1-\sqrt R)n$-close to the perturbed one.

Finally, fix any distribution $D_s$. It must hold that, with $1-\negl(\lambda)$ probability, there does not exist a valid message/signature pair where the message is $1-b$. If this does not hold, one may utilize this learning adversary $A$ to break the unforgeability of the signature scheme as follows: on input  the verification key $\vk$ and a valid message/signature $(b,\sign(\sk,b))$, the signature adversary samples the test instance and feed it to the adversary $A$, obtaining a perturbed instance $x'=(\widetilde{[u]},\widetilde{[v]},\widetilde{[\vk]},\widetilde{\mathsf{c}},\widetilde{\mathsf{d}})$.%
\footnote{Note that the adversary can efficiently sample from $D_s$ as the $(\vk,\sk)$ pairs for every instance are independent.}
The signature adversary uses the same procedure as the efficient learner to recover a list of message/signature pairs. If there is a valid message/signature pair with message $1-b$, clearly, the signature adversary breaks the unforgeability of the signature scheme.
Since the signature scheme is $\negl(\lambda)$-secure, it must hold that, with $1-\negl(\lambda)$ probability, there does not exist a valid message/signature pair where the message is $1-b$.

Consequently, this efficient learner outputs the correct label $b$ with $1-\negl(\lambda)$ probability. Thus, for all efficient adversary $A$,
$$\Ex_{ f \sim L(\emptyset, \lambda)}[ \Risk_{d,r}(h, D_s,f)]=\negl(\lambda),$$
and this finishes the proof.
\end{proof}

\ifnum\version=\neurips

\begin{theorem*}[Restatement of Theorem~\ref{thm:4-2'}]
Consider the learning task of Construction \ref{const:2'}. Then, for computationally unbounded adversaries, any information-theoretic learner with $\leq \alpha/2$ parameters cannot $\eps$-robustly learn $F_\lambda$ for $\eps<1/3$.
\end{theorem*}

\else
\subsection{Proof of Theorem~\ref{thm:4-2'}}
\fi

\begin{proof}
We sample $s$ uniformly at random from $\zo^\alpha$ and
prove that
$$\Ex_{\cS \sim D_s^n; f \sim L(\cS^h, \lambda)}[ \Risk_{d,r}(h, D_s,f)]>1/3.$$
The proof is similar to the proof of Theorem~\ref{thm:eff-learn'}. 

Let $Z$ denote $L(\cdot,\lambda)$, i.e., the parameters of the model output by the learner. 
Given a test instance $x=([u],[v],[\vk],\LEnc(b,\sign(\sk,b)),[b+\tuple{v,s|_{\samp(u)}}])$, we first prove the following claim.

\begin{claim}\label{clm:4-2-1}
With overwhelming probability over $Z$,
\begin{align*}
&\quad\bigg(Z,([u],[v],[\vk],\LEnc(b,\sign(\sk,b)),[b+\tuple{v,s|_{\samp(u)}}])\bigg)\\
\approx_{\negl(\lambda)}&\quad \bigg(Z,([u],[v],[\vk],\LEnc(b,\sign(\sk,b)),[U_{\zo}])\bigg).
\end{align*}    
\end{claim}
That is, the learner cannot distinguish the two distributions given $Z$.

First, we have $\widetilde{H}_\infty(s|Z)\geq \alpha/2$.
Define the set 
$$\mathsf{Good}=\{z\;:\; H_\infty(s|Z=z)\leq \alpha/4\}.$$
By Lemma~\ref{lem:min-ent}, $\prob{Z\in\mathsf{Good}}\leq 2^{-\alpha/4}=\negl(\lambda)$. For the rest of the analysis, we conditioned on $Z\notin\mathsf{Good}$.
Since $H_\infty(s|Z=z)>\alpha/4$, by the property of the sampler (Lemma~\ref{thm:samp}), there exists a distribution $D$ such that
$$H_\infty(D)\geq \left(\frac14-\kappa_1\right)\alpha >\frac15\alpha$$
and
\begin{equation}
\sd{(u,s|_{\samp(u)})\;,\; (u,D)}\leq \exp(-\Theta(\beta \kappa_1))+\exp\left(-n^{\kappa_2}\right)=\negl(\lambda).\label{eq:2}
\end{equation} 
Finally, by Theorem~\ref{thm:ip}, we have
\begin{align*}
&\;\sd{\left(v,\tuple{v,s|_{\samp(u)}}\right)\;,\; \left(v,U_{\zo}\right)}\\
\leq&\;\sd{\left(v,\tuple{v,s|_{\samp(u)}}\right)\;,\; \left(v,\tuple{v,D}\right)}+\sd{\left(v,\tuple{v,D}\right)\;,\; \left(v,U_{\zo}\right)}\tag{Triangle inequality}\\
\leq&\; \negl(\lambda)+\negl(\lambda).\tag{Equation~\ref{eq:2} and Theorem~\ref{thm:ip}}
\end{align*}
This completes the proof of Claim~\ref{clm:4-2-1}.

Now, consider the following adversary $A$ that perturbs $n/2$ symbols and does the following.
(Observe that $n/2<(1-\sqrt R)n$ for $R<1/4$ and, hence, the adversary is within budget.)
\begin{enumerate}
    \item $A$ decodes $\LEnc(b,\sign(\sk,b))$ to find $b$. Let $\bar b=1-b$. It forges a valid signature $ \sigma=\sign(\sk,\bar b)$ and encode it $\LEnc(\bar b,\sigma)$.
    \item Now, for a random subset $\cT\subseteq \{1,2,\ldots,n\}$ of size $\abs{\cT}=n/2$, $A$ replaces $\LEnc(b,\sign(\sk,b))$ with $\LEnc(\bar b,\sigma)$ on those $i\in \cT$.
    Then, after perturbation, the string $\LEnc(b,\sign(\sk,b))$ becomes a {\em random string} that has Hamming distance exactly $n/2$ from both $(0,\sign(\sk,0))$ and $(1,\sign(\sk,1))$. Let us call this distribution $X$. Note that $X$ is independent of $b$.
\end{enumerate}

Therefore, after the perturbation, the perturbed instance is statistically close to
$$([u],[v],[\vk],X,[U_{\zo}]),$$
which is independent of $b$.
Hence, the learner's output will not agree with $b$ with probability $\geq 1/2-\negl(\lambda)$.
Putting everything together, we have
\begin{align*}
&\;\Ex_{\cS \sim D_s^n; f \sim L(\cS^h, \lambda)}[ \Risk_{d,r}(h, D_s,f)]\\
\geq&\; \Pr_{\cS \sim D_s^n; f \sim L(\cS^h, \lambda)}[ z\in \mathsf{Good}] \\
&\; +\Pr_{\cS \sim D_s^n; f \sim L(\cS^h, \lambda)}[ z\notin \mathsf{Good}]\cdot \Ex_{\cS \sim D_s^n; f \sim L(\cS^h, \lambda)}[ \Risk_{d,r}(h, D_s,f)\vert  z\notin \mathsf{Good}] \\
\geq&\; \negl(\lambda)+(1- \negl(\lambda))\cdot \left(\frac12 
 - \negl(\lambda)\right) >1/3. &\qedhere
\end{align*}
\end{proof}

\section{Acknowledgement} Sanjam Garg and Mingyuan Wang are supported by DARPA under Agreement No. HR00112020026, AFOSR Award FA9550-19-1-0200, NSF CNS Award 1936826, and research grants by the Sloan Foundation, and Visa Inc. Any opinions, findings and conclusions or recommendations expressed in this material are those of the author(s) and do not necessarily reflect the views of the United States Government or DARPA. Mohammad Mahmoody is supported by NSF grants CCF-1910681 and CNS1936799.
Somesh Jha is partially supported by Air Force Grant FA9550-18-1-0166, the National Science Foundation (NSF) Grants CCF-FMitF-1836978, IIS-2008559, SaTC-Frontiers-1804648, CCF-2046710 and CCF-1652140, and ARO grant number W911NF-17-1-0405. Somesh Jha is also partially supported by the DARPA-GARD problem under agreement number 885000.

\bibliographystyle{plainnat}
\bibliography{ref,abbrev0,crypto}

\appendix

\section{Supplementary  material}
\label{app:add-prelim}

\ifnum\version=\neurips
\subsection{Learning}

Our proof also relies on the following theorem from \cite{bubeck2019adversarial}.  

\begin{theorem}[Implied by Theorem 3.1 of \cite{bubeck2019adversarial}]\label{thm:learnable}
 Let $\{\mathcal{C}_\lambda\}_\lambda$ be a finite family of classifiers. 
Suppose the learning problem $\cF=\{\cF_\lambda\}_\lambda$ and a learner $L $ satisfy the following. For all $D_\lambda\in\cD_\lambda$ and $h \in\cH_\lambda$, and sample $\cS \gets D^n_\lambda$, $L(\cS^h)$ always outputs a classifier   $f\in\mathcal{C}_\lambda$ such that $\Risk_{d,r}(h,D_\lambda,f)=0$ (i.e., $f$ robustly fits $h$ perfectly). Then, $L$ will $\delta$-robust learn $\cF$ with sample complexity $ \log\abs{\mathcal{C}_\lambda}/\delta$.
\end{theorem}
We emphasize that in the theorem above, one might pick a \emph{different}   set of classifiers merely for sake of computational efficiency of the learner $L$. Namely, it might be possible to information-theoretically learn a hypothesis class robustly (e.g., by a robust variant of empirical risk minimization when), but an efficient learner might choose to output its classifiers from a larger set such that it can \emph{efficiently} find a member of that class.
\else
\fi

\subsection{Cryptographic primitives}

\begin{definition}[Computational indistinguishability]
We say two ensembles of distributions $X=\{X_\lambda\}_{\lambda\in\mathbb{N}}$ and $Y=\{Y_\lambda\}_{\lambda\in\mathbb{N}}$ are computationally indistinguishable if for any probabilistic polynomial-time (PPT) algorithm $A$, it holds that 
$$\abs{\probsub{x\draw X_\lambda}{A(x)=1}-\probsub{y\draw Y_\lambda}{A(y)=1}}=\negl(\lambda).$$
\end{definition}

\begin{definition}[One-way function]\label{def:owf}
An ensemble of functions $\{f_\lambda:\zo^\lambda\to\zo^\lambda\}_\lambda$ is called a one-way function if for all polynomial-time probabilistic algorithm $A$, it holds that
$$\prob{\begin{gathered}
x\draw\zo^\lambda,\ 
y = f_\lambda(x)\\
x'\draw A(1^\lambda, y)
\end{gathered}  \;:\; f_\lambda(x')=y}=\negl(\lambda).$$
\end{definition}

\subsection{Coding theory}

\begin{fact}\label{fact:RS} The following facts hold about the Reed-Solomon code.
\begin{itemize}
    \item The distance of the Reed-Solomon code is $d = n-k+1$. Moreover, the decoding is possible efficiently:  there is a PPT   algorithm that maps any erroneous codeword that contains up to $\leq(n-k)/2$ errors to the nearest correct (unique) codeword. 
    In other words, one can efficiently correct up to $\frac{1-R}{2}$ fraction of errors, where $R$ is the code's rate.  
    
    \item The encoding of a {\em random} message is $k$-wise independent. That is, for all subset $\cS\subseteq \{1,2,\ldots,n\}$ such that $\abs{\cS}\leq k$, the following distribution 
    $$\left\{ \begin{gathered}
        m\draw\bbF^k,\ \mathsf{c} = m\cdot G\\
        \text{Output }\mathsf{c}_\cS
    \end{gathered}\right\}$$
    is uniform over $\bbF^{\abs{\cS}}$. This follows from the fact that any $\leq k$ columns of the generator matrix of the RS code is full-rank. 
\end{itemize}
\end{fact}

\ifnum\version=\neurips

\subsection{Randomness Extraction}

\else
\fi

\section{Missing Proofs}\label{app:missing}

\allowdisplaybreaks

\subsection{Proof of Theorem~\ref{thm:ip}}\label{app:proof-ip}
The theorem follows from the following derivation.
\begin{align*}
&\;\sd{\;\Big(Y,\tuple{X,Y}\Big)\;,\; \Big(Y,U_{\zo}\Big)}\\
=&\;\expsub{y\draw Y}{\sd{\;\tuple{X,y}\;,\;U_{\zo}}\;}\\
=&\;\frac12\cdot \expsub{y\draw Y}{\abs{\prob{\tuple{X,y}=0}-\prob{\tuple{X,y}=1}}}\\
\leq&\;\frac12\cdot \sqrt{\expsub{y\draw Y}{\left( 
\prob{\tuple{X,y}=0}-\prob{\tuple{X,y}=1} \right)^2}}\tag{Jensen's inequality}\\
=&\;\frac12\cdot\sqrt{\expsub{y\draw Y}{\probsub{x,x'\draw X}{\tuple{x,y}=\tuple{x',y}}-\probsub{x,x'\draw X}{\tuple{x,y}\neq \tuple{x',y}}}}\\
=&\;\frac12\cdot\sqrt{\probsub{x,x'\draw X}{\probsub{y\draw Y}{\tuple{x-x',y}=0}-\probsub{y\draw Y}{\tuple{x-x',y}=1}}}\\
=&\;\frac12\cdot\sqrt{\probsub{x,x'\draw X}{x = x'}\cdot \frac12}\tag{when $x\neq x'$, the inner term is always 0}\\
=&\;\frac12\cdot\sqrt{\frac12\cdot \sum_\omega(\prob{X=\omega})^2}\\
\leq &\;\frac12\cdot\sqrt{\frac12\cdot \sum_\omega\prob{X=\omega}\cdot 2^{-H_\infty(X)}}\tag{By definition of min-entropy}\\
=&\;\frac12\cdot\sqrt{ 2^{-H_\infty(X)-1}}\leq \eps
\end{align*}

\subsection{Proof of Theorem~\ref{thm:masking}}\label{app:masking-proof}

Dodis and Smith~\cite{STOC:DodSmi05} proved this theorem for $\bbF_2$. We are simply revising their proof for the field $\bbF_{2^\ell}$.
Within this proof, we shall use $\bbF$ for $\bbF_{2^\ell}$.
We need the following claims.

\begin{claim}[Parseval's identity]
$\sum_{\alpha} \bias(X,\alpha)^2 =  {\abs{\bbF}^n}\cdot\sum_\omega(\prob{X=\omega})^2.$
\end{claim}

\begin{proof}
Observe that
\begin{align*}
&\sum_{\alpha}\bias(X,\alpha)^2\\
=&\;\sum_\alpha\left(\expsub{x\draw X}{(-1)^{\tr{\tuple{x,\alpha}}}}\right)^2\\
=&\; \sum_\alpha\expsub{x,x'\draw X}{(-1)^{\tr{\tuple{x,\alpha}}}\cdot (-1)^{\tr{\tuple{x',\alpha}}}}\\
=&\; \sum_\alpha\expsub{x,x'\draw X}{(-1)^{\tr{\tuple{x+x',\alpha}}}}\tag{Since the trace map is additive}\\
=&\; \expsub{x,x'\draw X}{\sum_\alpha(-1)^{\tr{\tuple{x+x',\alpha}}}}\\
=&\; \abs{\bbF}^n\probsub{x,x'\draw X}{x=x'}.
\end{align*}
Here, we use the fact that, when $x\neq x'$, the inner term is 0 as the trace map maps half of the field to 0 and the other half to 1.

Note that the last line is exactly equal to
\[
\abs{\bbF}^n\cdot\sum_\omega(\prob{X=\omega})^2.  \qedhere
\]
\end{proof}

\begin{claim}[Bias of Convolusion is product of bias]\label{clm:convol}
$\bias(X+Y,\alpha) =  \bias(X,\alpha)\cdot\bias(Y,\alpha).$
\end{claim}

\begin{proof}
Observe that
\begin{align*}
&\bias(X+Y,\alpha)\\
=&\; \sum_{\omega}\prob{X+Y=\omega}\cdot {(-1)^{\tr{\tuple{\omega,\alpha}}}}\\
=&\; \sum_\omega\sum_{\omega'}\prob{X=\omega'}\prob{Y=\omega-\omega'}\cdot {(-1)^{\tr{\tuple{\omega,\alpha}}}}\\
=&\; \sum_{\omega''}\sum_{\omega'}\prob{X=\omega'}\prob{Y=\omega''}\cdot {(-1)^{\tr{\tuple{\omega'+\omega'',\alpha}}}}\\
=&\; \left(\sum_{\omega'}\prob{X=\omega'}\cdot {(-1)^{\tr{\tuple{\omega',\alpha}}}}\right)\cdot \left(\sum_{\omega''}\prob{Y=\omega''}\cdot {(-1)^{\tr{\tuple{\omega'',\alpha}}}}\right)\\
=&\; \bias(X,\alpha)\cdot \bias(Y,\alpha).&&\qedhere
\end{align*}
\end{proof}

Given these two claims, we prove the theorem as follows.

\begin{align*}
& \sd{X+Y,U_{\bbF^n}}\\
=&\; \frac12\cdot \sum_{\omega}\abs{\prob{X+Y=\omega}-\prob{U_{\bbF^n}=\omega}}\\
\leq&\; \frac12\cdot\sqrt{\abs{\bbF}^n\cdot\sum_\omega \left(\prob{X+Y=\omega}-\prob{U_{\bbF^n}=\omega}\right)^2}\tag{Cauchy-Schwartz}\\
=&\; \frac12\cdot \sqrt{ \sum_{\alpha}\left(\bias(X+Y,\alpha)-\bias(U_{\bbF^n},\alpha)\right)^2}\tag{Parseval}\\
=&\;\frac12\cdot \sqrt{ \sum_{\alpha\neq 0^n}\bias(X+Y,\alpha)^2}\tag{Since $\bias(U_{\bbF^n},\alpha)=0$ for all $\alpha\neq 0^n$.}\\
=&\; \frac12\cdot\sqrt{ \sum_{\alpha\neq 0^{n}}\bias(X,\alpha)^2\cdot\bias(Y,\alpha)^2}\tag{By Claim~\ref{clm:convol}}\\
\leq&\; \frac\eps2\cdot\sqrt{ \sum_{\alpha\neq 0^n}\bias(X,\alpha)^2}\tag{Since $Y$ is small-biased}\\
\leq&\; \frac\eps2\cdot\sqrt{\abs{\bbF}^{n}\sum_{\omega}(\prob{X=\omega})^2}\tag{Parseval}\\
=&\; \frac\eps2\cdot\sqrt{\abs{\bbF}^n\sum_\omega\prob{X=\omega}\cdot2^{-H_\infty(X)}}\tag{Definition of min-entropy}\\
=&\; \frac\eps2\cdot\sqrt{\abs{\bbF}^n \cdot2^{-H_\infty(X)}}\\
=&\; 2^{\frac{n\ell-k}{2}-1}\cdot \eps
\end{align*}

\subsection{Proof of Theorem~\ref{thm:noisy-code}}\label{app:noisy-proof}

We divide all possible linear tests $\alpha$ into two cases.
\begin{itemize}
    \item {\bfseries Small linear tests are fooled by RS code.} We say that $\alpha$ is a small linear test if ${\abs{\{i\colon \alpha_i\neq 0\}}\leq Rn}$. By Fact~\ref{fact:RS}, a random codeword  projects onto any $\leq Rn$ coordinates is always a uniform distribution. Hence, $\tuple{D,\alpha}$ is always uniform. Consequently, $\bias( D,\alpha)=0$ for all small linear test.  
    \item {\bfseries Large linear tests are fooled by the noise.} Suppose $\alpha$ is such that $\abs{\cT}> Rn$, where $\cT=\{i\colon \alpha_i\neq 0\}$. Observe that $\cS$ is a random subset of size $s$ and $\cT$ is a fixed set of size $> Rn$. Clearly, $\cS\cap \cT=\emptyset$ happens with probability $\leq (1-R)^s$. Now, conditioned on the event that $\cS\cap \cT\neq \emptyset$, we again have $\tuple{D,\alpha}$ is a uniform distribution (because of the random noise). Consequently, for large $\alpha$, we have $\bias(D,\alpha)\leq (1-R)^s$.
\end{itemize}
Therefore, for all possible $\alpha$, $\bias(D,\alpha)$ is small. Hence, the theorem follows.

\else       
\documentclass{article}


\usepackage[final]{neurips_2022}

\usepackage{algorithm}
\usepackage{lipsum}

\usepackage{algorithmic}




\usepackage[utf8]{inputenc} 
\usepackage[T1]{fontenc}    
\usepackage[colorlinks,linkcolor=blue,urlcolor=blue,citecolor=blue]{hyperref}       
\usepackage{url}            
\usepackage{booktabs}       
\usepackage{amsfonts}       
\usepackage{nicefrac}       
\usepackage{microtype}      
\usepackage{xcolor} 



%

\begin{document}

\author{
Sanjam Garg\thanks{UC Berkeley and NTT Research \href{mailto:sanjamg@berkeley.edu}{sanjamg@berkeley.edu}} \And 
Somesh Jha\thanks{University of Wisconsin, Madison \href{mailto:jha@cs.wisc.edu}{jha@cs.wisc.edu}} \And 
Saeed Mahloujifar\thanks{Princeton University \href{mailto:sfar@princeton.edu}{sfar@princeton.edu}}  \AND 
Mohammad Mahmoody\thanks{University of Virginia \href{mailto:mohammad@virginia.edu}{mohammad@virginia.edu}} \And
Mingyuan Wang\thanks{UC Berkeley \href{mailto:mingyuan@berkeley.edu}{mingyuan@berkeley.edu}}
}

\maketitle



\section{Acknowledgement} Sanjam Garg and Mingyuan Wang are supported by DARPA under Agreement No. HR00112020026, AFOSR Award FA9550-19-1-0200, NSF CNS Award 1936826, and research grants by the Sloan Foundation, and Visa Inc. Any opinions, findings and conclusions or recommendations expressed in this material are those of the author(s) and do not necessarily reflect the views of the United States Government or DARPA. Mohammad Mahmoody is supported by NSF grants CCF-1910681 and CNS1936799.
Somesh Jha is partially supported by Air Force Grant FA9550-18-1-0166, the National Science Foundation (NSF) Grants CCF-FMitF-1836978, IIS-2008559, SaTC-Frontiers-1804648, CCF-2046710 and CCF-1652140, and ARO grant number W911NF-17-1-0405. Somesh Jha is also partially supported by the DARPA-GARD problem under agreement number 885000.


\bibliographystyle{plainnat}
\bibliography{ref,abbrev0,crypto}

\appendix

\input{checklist}

\newpage

\ifnum\temp=1
\section{Revision}

\paragraph{Limitations, Implications, and Open Question.}
Our work shows that the phenomenon of having larger models due to computational efficiency could provably happen in certain scenarios. It does not imply, however, that this holds for all learning problems.
It is a fascinating open question whether similar phenomena also happen to real-world problems.
We note that this is not particular to our work, but common to most prior works in the theory of learning showing ``separation'' results~\citep{bubeck2019adversarial, degwekar2019computational,mahloujifar2019can,garg2020adversarially}.

Our results provides an explanation on why small but representative classes (e.g. 2 layers neural networks) of functions do not obtain same (robust) accuracy as larger models. This phenomenon that cannot be solely explained based on representation power of the function class might be due to computational limitations of the learning algorithm.

Finally, we note that our theorem demonstrates the separation by using the simplest setting of binary output. One can extend it to more sophisticated settings of any finite output, particularly real numbers with bounded precision. However, our work does not consider real numbers with infinite precision. In such cases, one needs to revisit computational efficiency, as the inputs are infinitely long. 
\fi

\fi

\end{document}